\documentclass[10pt,journal,compsoc]{IEEEtran}
\IEEEoverridecommandlockouts
\usepackage{graphicx}
\usepackage{array}
\usepackage{float}
\usepackage{url}
\usepackage{xcolor}
\usepackage{float}
\usepackage{lipsum}
\usepackage{todonotes}
\newcounter{todocounter}

\usepackage{epsfig}
\usepackage{amsmath}
\usepackage{physics}
\usepackage{amssymb}
\usepackage{bbm}

\usepackage{bm}
\usepackage{amsfonts,multirow,bigstrut,booktabs,ctable,latexsym}
\usepackage{mathtools}
\usepackage[mathscr]{eucal}
\usepackage{verbatim}
\usepackage{amsthm}
\usepackage{multicol}
\usepackage{soul}
\usepackage{cite}
\usepackage{float}
\usepackage{flafter}

\restylefloat{table}

\usepackage{algorithm}
\usepackage{algcompatible}

\usepackage{caption}
\ifCLASSOPTIONcompsoc
\usepackage[caption=false,font=normalsize,labelfont=sf,textfont=sf]{subfig}
\else
\usepackage[caption=false,font=footnotesize]{subfig}
\fi
 
\newtheorem{theorem}{Theorem}

\newtheorem{assumption}{Assumption}
\newtheorem{lemma}{Lemma}
\newtheorem{definition}{Definition}

\makeatletter
\def\ALG@special@indent{%
    \ifdim\ALG@thistlm=0pt\relax
        \hskip-\leftmargin
    \else
        \hskip\ALG@thistlm
    \fi
}

\newcommand{\Serverexe}[1]{\item[]\noindent\ALG@special@indent  \textbf{Server executes:}\ #1}
\newcommand{\Clientexe}[1]{\item[]\noindent\ALG@special@indent \textbf{ClientUpdate($\bm{\theta}^{t}, \bm{m}^t, \spar$):}\ #1}
\newcommand{\Clientexedpfed}[1]{\item[]\noindent\ALG@special@indent \textbf{ClientUpdate($\bm{\theta}^{t}$):}\ #1}
\newcommand{\Selecttopk}[1]{\item[]\noindent\ALG@special@indent \textbf{SelectTopk($\bm{\theta}^t, k$):}\ #1}
\newcommand{\Selectrandk}[1]{\item[]\noindent\ALG@special@indent \textbf{SelectRandk($\bm{\theta}^t, k$):}\ #1}
\newcommand{\DPspar}[1]{\item[]\noindent\ALG@special@indent \textbf{DP-spar($\Delta, \bm{m}^t, \spar$):}\ #1}
\makeatother

\makeatletter

\newcommand{\returnx}{\item\noindent{return}\ }

\makeatother

\DeclareMathOperator{\topk}{top} 
\DeclareMathOperator{\randk}{rand} 
\DeclareMathOperator{\spar}{spar} 

\def\BibTeX{{\rm B\kern-.05em{\sc i\kern-.025em b}\kern-.08em
    T\kern-.1667em\lower.7ex\hbox{E}\kern-.125emX}}
\begin{document}

%\title{Fed-SMP: Private and Efficient Federated Learning with Privacy-Coded Sparsification and Momentum
%\title{The Value of Sparsification in Federated Learning: Communication, Privacy, and Accuracy
%\title{RDP-FedAvg: R{\'e}nyi Differentially Private Federated Learning with Sparsified Communication 
% \title{Fed-SMP: Communication-Efficient Federated Learning with Local Differential Privacy 
%  under Client-Level Differential Privacy Guarantee
% %\thanks{Identify applicable funding agency here. If none, delete this.}
% }

\title{Federated Learning with Sparsified Model Perturbation: Improving Accuracy under Client-Level Differential Privacy}
% \title{On the Convergence of Private Federated Learning without a Trusted Server}
\author{Rui~Hu,~\IEEEmembership{Member,~IEEE,} Yuanxiong~Guo,~\IEEEmembership{Senior~Member,~IEEE} and~Yanmin~Gong,~\IEEEmembership{Senior~Member,~IEEE}
\IEEEcompsocitemizethanks{\IEEEcompsocthanksitem R. Hu is with the Department of Computer Science \& Engineering, University of Nevada, Reno, Reno, NV, 89557 USA (E-mail: ruihu@unr.edu); Y. Guo is with the Department of Information Systems and Cyber Security, University of Texas at San Antonio, San Antonio, TX, 78249 USA (E-mail: yuanxiong.guo@utsa.edu); Y. Gong is with the Department of Electrical and Computer Engineering, University of Texas at San Antonio, San Antonio, TX, 78249 USA (E-mail: yanmin.gong@utsa.edu). 
}

% \IEEEcompsocitemizethanks{\IEEEcompsocthanksitem R. Hu is with the Department of Computer Science \& Engineering, University of Nevada, Reno, Reno, NV, 89557 USA (E-mail: ruihu@unr.edu).
% \IEEEcompsocthanksitem Y. Gong is with the Department of Electrical and Computer Engineering, University of Texas at San Antonio, San Antonio, TX, 78249 USA (E-mail: yanmin.gong@utsa.edu).
% \IEEEcompsocthanksitem Y. Guo is with the Department of Information Systems and Cyber Security, University of Texas at San Antonio, San Antonio, TX, 78249 USA (E-mail: yuanxiong.guo@utsa.edu).
% }

% \thanks{Manuscript received April 19, 2005; revised August 26, 2015.}}

%   \author{Rui~Hu,~\IEEEmembership{Student Member,~IEEE,} Yanmin~Gong,~\IEEEmembership{Senior~Member,~IEEE}
%           and~Yuanxiong~Guo,~\IEEEmembership{Senior~Member,~IEEE}% <-this % stops a space
%  %\author{Rui~Hu, Yuanxiong~Guo, and Yanmin~Gong% <-this % stops a space
%  \thanks{Y. Gong are with the Department of Electrical and Computer Engineering, University of Texas at San Antonio, San Antonio,
%  TX, 78249 USA (e-mail: \{rui.hu, yanmin.gong\}@utsa.edu).}% <-this % stops a space
%  \thanks{Y. Guo is with the Department of Information Systems and Cyber Security, University of Texas at San Antonio, San Antonio, TX, 78249 USA (e-mail: yuanxiong.guo@utsa.edu).}% <-this % stops a space
% % %\thanks{Manuscript received April 19, 2005; revised August 26, 2015.}}
 }

\IEEEtitleabstractindextext{%
\begin{abstract}
%Federated learning (FL) has received explosive interests. By learning a shared statistical model over distributed clients, FL keeps data locally and greatly improves data privacy and communication efficiency. 
%
Federated learning (FL) that enables edge devices to collaboratively learn a shared model while keeping their training data locally has received great attention recently and can protect privacy in comparison with the traditional centralized learning paradigm. However, sensitive information about the training data can still be inferred from model parameters shared in FL. Differential privacy (DP) is the state-of-the-art technique to defend against those attacks. The key challenge to achieving DP in FL lies in the adverse impact of DP noise on model accuracy, particularly for deep learning models with large numbers of parameters. This paper develops a novel differentially-private FL scheme named Fed-SMP that provides a client-level DP guarantee while maintaining high model accuracy. To mitigate the impact of privacy protection on model accuracy, Fed-SMP leverages a new technique called Sparsified Model Perturbation (SMP) where local models are sparsified first before being perturbed by Gaussian noise. %Two sparsification strategies are considered in Fed-SMP: random sparsification and top-$k$ sparsification. 
We provide a tight end-to-end privacy analysis for Fed-SMP using R{\'e}nyi DP and prove the convergence of Fed-SMP with both unbiased and biased sparsifications. Extensive experiments on real-world datasets are conducted to demonstrate the effectiveness of Fed-SMP in improving model accuracy with the same DP guarantee and saving communication cost simultaneously. 

\end{abstract}
}
\maketitle
\IEEEdisplaynontitleabstractindextext
\IEEEpeerreviewmaketitle

%%%%%%%%%%%%%%%%%%%%%%%%%%%%%%%%%%%%%%%%%%%%%%%%%%%%
\IEEEraisesectionheading{
\section{Introduction}\label{sec:intro}}
%%%%%%%%%%%%%%%%%%%%%%%%%%%%%%%%%%%%%%%%%%%%%%%%%%%%
\IEEEPARstart{T}he proliferation of edge devices such as smartphones and Internet-of-things (IoT) devices, each equipped with rich sensing, computation, and storage resources, leads to tremendous data being generated on a daily basis at the network edge. These data can be analyzed to build machine learning models that enable a wide range of intelligent services such as personal fitness tracking \cite{iotwearable}, traffic monitoring \cite{iottraffic}, and smart home security \cite{smarthome}. Traditional machine learning paradigm requires transferring all the raw data to the cloud before training a model, leading to high communication cost and severe privacy risk. 

As an alternative machine learning paradigm, \emph{Federated Learning (FL)} has attracted significant attention recently due to its benefits in communication efficiency and privacy \cite{kairouz2019advances}. In FL, multiple clients collaboratively learn a shared statistical model under the orchestration of the cloud without sharing their local data \cite{mcmahan2017communication}. Although only model updates instead of raw data are shared by each client in FL, it is not sufficient to ensure privacy as the sensitive training data can still be inferred from the shared model parameters by using advanced inference attacks. 
For instance, given an input sample and a target model, the membership inference attack \cite{shokri2017membership} can train an attack model to determine whether the sample was used for training the target model or not. %by training an attack model to recognize the differences in the behaviors of the target model on inputs that were used to train the target model versus inputs that the target model did not encounter during training. 
Also, given a target model and class, the model inversion attack \cite{fredrikson2015model} can recover the typical representations of the target class using an inversion model learned from the correlation between the inputs and outputs of the target model.  
%
%These two kinds of attacks can directly lead to privacy breach. % in FL once the private local models are compromised.
% whether a client participates in the training or not, and model inversion attack \cite{fredrikson2015model} can reconstruct the training data. \hl{Need to be elaborated further in details. Determine/infer/construct what from what ....}

%In FL, such adversaries could be either the untrusted cloud or other parties who have access to the exchanged model updates. It is thus necessary to provide strong and rigorous privacy protection to client data in FL.

As a cryptography-inspired rigorous definition of privacy, differential privacy (DP) has become the de-facto standard for achieving data privacy and can give a strong privacy guarantee against an adversary with arbitrary auxiliary information \cite{dwork2014algorithmic}. For privacy protection in FL, client-level DP is often more relevant than record-level DP: client-level DP guarantee protects the participation of a client, while record-level DP guarantee protects only a single data sample of a client \cite{geyer2017differentially,abadi2016deep}. 
% we need to consider the more practical client-level DP setting where the participation of a client should be protected, rather than the record-level DP setting where a single data sample is protected. 
While DP can be straightforwardly achieved using Gaussian or Laplacian mechanism \cite{dwork2014algorithmic}, achieving client-level DP in the FL setting faces several major challenges in maintaining high model accuracy. Firstly, FL is an iterative learning process where model updates are exchanged in multiple rounds, leading to more privacy leakage compared to the one-shot inference. Secondly, the intensity of added DP noise is linearly proportional to the model size, which can be very large (e.g., millions of model parameters) for modern deep neural networks (DNNs), and will severely degrade the accuracy of the trained model. Thirdly, it is more challenging to achieve client-level DP than record-level DP because the entire dataset of a client rather than a single data sample of that client needs to be protected. Existing studies \cite{mcmahan2017learning,liu2020flame,erlingsson2019amplification} on FL with client-level DP suffer from significant accuracy degradation due to the inherent challenge of large additive random noise required to achieve a certain level of client-level DP for DNNs.  

%Existing works on differentially-private FL either assume a fully-trusted server \cite{mcmahan2017learning} or rely on shuffling via anonymous channels \cite{sun2020ldp,liu2020flame,girgis2020shuffled,erlingsson2019amplification,balle2020privacy} to improve the accuracy-privacy tradeoffs, which are not always feasible in practical FL settings. Balancing model accuracy and privacy protection in practical FL settings remains a big challenge.

In this paper, we propose a new differentially-private FL scheme called \emph{Fed-SMP}, which guarantees client-level DP while preserving model accuracy and saving communication cost simultaneously. Fed-SMP utilizes \emph{Sparsified Model Perturbation (SMP)} to improve the privacy-accuracy tradeoff of DP in FL. Specifically, SMP first sparsifies the local model update of a client at each round by selecting only a subset of coordinates to keep and then adds Gaussian noise to perturb the values at those selected coordinates. As we will show later in this paper, by using SMP in FL, the sparsification will have an amplification effect on the privacy guarantee offered by the added Gaussian noise, leading to a better privacy-accuracy tradeoff. Meanwhile, it can reduce the communication cost by compressing the shared model updates at each round. In summary, the main contributions of this paper are summarized as follows.
\begin{itemize}
    \item We propose a new differentially private FL scheme called Fed-SMP, which achieves client-level DP guarantee for FL with a large number of clients. Fed-SMP only makes lightweight modifications to FedAvg \cite{mcmahan2017communication}, the most common learning method for FL, which enables easy integration into existing packages. 

    \item Compared with DP-FedAvg\cite{mcmahan2017learning}, the state-of-art of differentially private FL schemes, Fed-SMP improves the privacy-utility tradeoff and communication efficiency of FL with client-level DP by using sparsification as a tool for amplifying privacy and reducing the communication cost at the same time. 

    \item By integrating different sparsification operators with model perturbation, we design two algorithms of Fed-SMP: Fed-SMP with unbiased random sparsification and Fed-SMP with biased top-$k$ sparsification. The resulting algorithms require less amount of added random noise to achieve the same level of DP and are compatible with secure aggregation, which is a crucial privacy-enhancing technique to achieve client-level DP in practical FL systems.

    \item To prove the $(\epsilon,\delta)$-DP guarantee of Fed-SMP, we use R{\'e}nyi differential privacy (RDP) to tightly account the end-to-end privacy loss of  Fed-SMP. We also theoretically analyze the impact of sparsified model perturbation on the convergence of Fed-SMP with both unbiased and biased sparsifications, filling the gap in the state-of-arts. The theoretical results indicate that under a certain DP guarantee, the optimal compression level of Fed-SMP needs to balance the increased compression error and reduced privacy error due to sparsification. 
    
    \item We empirically evaluate the performances of Fed-SMP on both IID and non-IID datasets and compare the results with those of the state-of-art baselines. Experimental results demonstrate that Fed-SMP can achieve higher model accuracy than baseline approaches under the same level of DP and save the communication cost simultaneously. 

    % \item We propose a new differentially-private FL scheme called Fed-SMP to improve the privacy-utility tradeoff of FL with client-level DP. By using sparsification as a tool for amplifying privacy and reducing communication cost at the same time, the proposed method can achieve both goals of privacy protection and communication efficiency at little cost to model accuracy. Fed-SMP only makes lightweight modifications to FedAvg \cite{mcmahan2017communication}, the most common learning method for FL, which enables easy integration into existing packages. 

    % \item We design two algorithms of Fed-SMP based on different sparsification operators: Fed-SMP with random sparsification and Fed-SMP with top-$k$ sparsification. The resulting algorithms require less amount of added random noise to achieve the same level of DP and are compatible with secure aggregation, which is a crucial privacy-enhancing technique to achieve client-level DP in practical FL systems. 

    % \item We theoretically analyze the impact of model sparsification on the convergence of Fed-SMP under $(\epsilon,\delta)$-DP guarantee and show that the optimal compression level needs to balance the increased compression error and reduced privacy error due to sparsification. 
    
    % \item We empirically evaluate the performances of the proposed schemes on Fashion-MNIST and SVHN datasets and compare the results with those of the state-of-art baselines. Experimental results demonstrate that Fed-SMP can achieve higher model accuracy than baseline approaches under the same level of DP while saving the communication cost. 
\end{itemize}

The rest of the paper is organized as follows. Preliminaries on DP are described in Section~\ref{sec:pre}. Section~\ref{sec:sys-mod} introduces the problem formulation and presents the proposed Fed-SMP scheme. The privacy and convergence properties of Fed-SMP with random and top-$k$ sparsification strategies are rigorously analyzed in Section~\ref{sec:main_results}. Section~\ref{sec:exp} shows the experimental results. Finally, Section~\ref{sec:related} reviews the related work, and Section~\ref{sec:con} concludes the paper. 

\section{Preliminaries}\label{sec:pre}
%%%%%%%%%%%%%%%%%%%%%%%%%%%%%%%%%%%%%%%%%%%%%%%%%%%%
% \subsection{DP in Machine Learning}\label{sec:DP_definition}
DP has been proposed as a rigorous privacy notion for measuring privacy risk. The classic notion of DP, $(\epsilon, \delta)$-DP, is defined as follows:
%-------------------------------------------%
\begin{definition}[$(\epsilon,\delta)$-DP\cite{dwork2014algorithmic}]\label{DP} 
Given privacy parameters $\epsilon >0$ and $0\leq \delta < 1$, a randomized mechanism $\mathcal{M}$ satisfies $(\epsilon,\delta)$-DP if for any two adjacent datasets $D, D^{\prime}$ and any subset of outputs ${O} \subseteq \text{range}(\mathcal{M})$,
\begin{equation}
\Pr[\mathcal{M}(D) \in {O}] \leq e^{\epsilon} \Pr[\mathcal{M}(D^{\prime}) \in {O}] + \delta.
\end{equation}
When $\delta=0$, we have $\epsilon$-DP, or Pure DP.
\end{definition}
%-------------------------------------------%

% \gong{Should be client-level DP. Thus two adjacent dataset differ by replacing one user's data. Introduce adjacent datasets before DP definition as RDP would also use it. }

%The concept of DP has been applied to deep learning algorithms in \cite{abadi2016deep} with the goal of ensuring that a model does not reveal whether a data sample has been used for training. 

To better calculate the privacy loss over multiple iterations in differentially private learning algorithms, R{\'e}nyi differential privacy (RDP) has been proposed as follows:
%-------------------------------------------%
\begin{definition}[$(\alpha, \rho)$-RDP  \cite{mironov2017renyi}]\label{rdp}
Given a real number $\alpha>1$ and privacy parameter $\rho\geq 0$, a randomized mechanism $\mathcal{M}$ satisfies $(\alpha, \rho)$-RDP if for any two adjacent datasets $D, D^{\prime}$, the R{\'e}nyi $\alpha$-divergence between $\mathcal{M}(D)$ and $\mathcal{M}(D^{\prime})$ satisfies
\begin{equation*}
D_{\alpha}[\mathcal{M}(D)\|\mathcal{M}(D^{\prime})]:= \frac{1}{\alpha-1}\log\mathbb{E}\left[ \left(\frac{{\mathcal{M}(D)}}{{\mathcal{M}(D^\prime)}}\right)^{\alpha}\right]\leq \rho(\alpha).
\end{equation*}
\end{definition}
%-------------------------------------------%
RDP is a relaxed version of pure DP with a tighter composition bound. Thus, it is more suitable to analyze the end-to-end privacy loss of iterative algorithms. We can convert RDP to $(\epsilon, \delta)$-DP for any $\delta>0$ using the following lemma:
%-------------------------------------------%
\begin{lemma}[From RDP to $(\epsilon, \delta)$-DP \cite{wang2019subsampled}]\label{lemma:rdp_dp}
If the randomized mechanism $\mathcal{M}$ satisfies $(\alpha, \rho(\alpha))$-RDP, then it also satisfies $(\rho(\alpha) + \frac{\log(1/\delta)}{\alpha-1}, \delta)$-DP.
\end{lemma}
%-------------------------------------------%

% RDP is also connected to variants of the Concentrated Differential Privacy (CDP) \cite{dwork2016concentrated,bun2016concentrated}, which can be thought of as a linear upper bound of the RDP function $\rho(\alpha)$. The following lemma allows us to provide a RDP bound for a pure DP mechanism:
% %-------------------------------------------%
% \begin{lemma}[From $\epsilon$-DP to RDP \cite{bun2016concentrated,zhu2020improving}]\label{lemma:dp-rdp}
% If the randomized mechanism $\mathcal{M}$ satisfies $\epsilon$-DP, then it also obeys $(\alpha, \rho(\alpha))$-RDP with $\rho(\alpha)=1/(\alpha-1)\log((\sinh(\alpha\rho)-\sinh((\alpha-1)\rho))/\sinh(\rho))\leq \alpha\rho^2/2$
% \end{lemma}
% %-------------------------------------------%

% The most common mechanisms for DP are those that perturb queries answers by adding noise.
% %-------------------------------------------%
% \begin{definition}[Noise-addition Mechanism]
% Let $h: \mathcal{D} \rightarrow \mathbb{R}$ be a query function over datasets. The noise-addition mechanism answers a query $h$ by outputting $o \sim \mathcal{M}(h) =  h(D) + Z$ where $Z$ is a random variable.
% \end{definition}
% %-------------------------------------------%
% Typical examples of the noise-addition mechanism for DP includes Laplace mechanism and Gaussian mechanism in which $Z$ is drawn from a Laplace distribution and a Gaussian distribution, respectively.

In the following, we provide some useful definitions and lemmas about DP and RDP that will be used to derive our main results in the rest of the paper.
%-------------------------------------------%
\begin{definition}[$\ell_2$-sensitivity\cite{dwork2014algorithmic}]\label{def:sensitivity}
Let $h: \mathcal{D} \rightarrow \mathbb{R}^d$ be a query function over a dataset. The $\ell_2$-sensitivity of $h$ is defined as $\psi(h):=\sup_{D,D^\prime \in \mathcal{D}, D \sim D^\prime}\|h(D)-h(D^\prime)\|_2$ where $D \sim D^\prime$ denotes that $D$ and $D^\prime$ are two adjacent datasets.  
\end{definition}
%-------------------------------------------%
\begin{lemma}[Gaussian Mechanism \cite{mironov2017renyi}]\label{lemma:gaussian_mechanism}
Let $h: \mathcal{D} \rightarrow \mathbb{R}^d$ be a query function with $\ell_2$-sensitivity $\psi(h)$. The Gaussian mechanism $\mathcal{M} = h(D) + \mathcal{N}(0, \sigma^2 \psi(h)^2 \bm{I}_d)$ satisfies $(\alpha,\alpha /2\sigma^2)$-RDP.
\end{lemma}
% -------------------------------------------%
% \begin{lemma}[Laplace Mechanism \cite{mironov2017renyi}]\label{lemma:laplace_mechanism}
% The Laplace mechanism $\mathcal{M} =  h(D) + Z$ with $Z\sim \lap(0,b)$ satisfies $\psi^2(h)/b$-DP.
% \end{lemma}
% -------------------------------------------%
% \begin{lemma}[Exponential Mechanism \cite{mironov2017renyi}]\label{lemma:laplace_mechanism}
% The Laplace mechanism $\mathcal{M} =  h(D) + Z$ with $Z\sim Exp()$ satisfies $\psi^2(h)/b$-DP.
% \end{lemma}
% -------------------------------------------%
\begin{lemma}[RDP for Subsampling Mechanism \cite{wang2019subsampled,wang2019efficient}]\label{lemma:rdp_sub}
For a Gaussian mechanism $\mathcal{M}$ and any m-datapoints dataset $D$, define $\mathcal{M}\circ \textit{SUBSAMPLE}$ as 1) subsample without replacement $B$ datapoints from the dataset (denote $q=B/m$ as the sampling ratio); and 2) apply $\mathcal{M}$ on the subsampled dataset as input. Then if $\mathcal{M}$ satisfies $(\alpha,\rho(\alpha))$-RDP with respect to the subsampled dataset for all integers $\alpha \geq 2$, then the new randomized mechanism $\mathcal{M}\circ \textit{SUBSAMPLE}$ satisfies $(\alpha,\rho^\prime(\alpha))$-RDP w.r.t $D$, where
\begin{multline*}
\rho^\prime(\alpha) \leq \frac{1}{\alpha-1}\log \bigg(1 + q^2 {\alpha \choose 2}\min\{4(e^{\rho(2)}-1), 2e^{\rho(2)}\}\\ 
+ \sum_{j=3}^{\alpha}q^j{\alpha \choose j}2e^{(j-1)\rho(j)}\bigg).
\end{multline*}
If $\sigma^2 \geq 0.7$ and $\alpha \leq (2/3)\sigma^2\psi^2(h) \log(1/q\alpha(1+\sigma^2)) + 1$, $\mathcal{M}\circ \textit{SUBSAMPLE}$ satisfies $(\alpha, 3.5 q^2 \alpha / \sigma^2)$-RDP. 
\end{lemma}
%-------------------------------------------%
\begin{lemma}[RDP Composition \cite{mironov2017renyi}]\label{lemma:rdp_comr} 
For randomized mechanisms $\mathcal{M}_1$ and $\mathcal{M}_2$ applied on dataset $D$, if $\mathcal{M}_1$ satisfies $(\alpha, \rho_1)$-RDP and $\mathcal{M}_2$ satisfies $(\alpha, \rho_2)$-RDP, then their composition $ \mathcal{M}_1 \circ \mathcal{M}_2$ satisfies $(\alpha, \rho_1+\rho_2)$-RDP.
\end{lemma}
\section{System Modeling and Problem Formulation}\label{sec:sys-mod}
%%%%%%%%%%%%%%%%%%%%%%%%%%%%%%%%%%%%%%%%%%%%%%%%%%%%

In this section, we first present the problem formulation of FL and the attack model, then describe the classic method to solve the problem and its privacy-preserving variant. %Next we describe the design of our proposed Fed-SMP scheme, which is both differentially private and communication-efficient. 

%%%%%%%%%%%%%%%%%%%%%%%%%%%%%%%%%%%%%%%%%%%%%%%%%%%%
\subsection{Problem Formulation}\label{subsec:form_fed}
%%%%%%%%%%%%%%%%%%%%%%%%%%%%%%%%%%%%%%%%%%%%%%%%%%%%

A typical FL system consists of $n$ clients (e.g., smartphones or IoT devices) and a central server (e.g., the cloud), as shown in Fig.~\ref{fig:fl_system}. Each client $i\in[n]$ has a local dataset $D_i$, and those clients collaboratively train a global model $\bm{\theta} \in \mathbb{R}^d$ based on their collective datasets under the orchestration of the central server. The goal of FL is to solve the following optimization problem:
% \begin{align}\label{fed_obj}
% \min_{\bm{\theta}} f(\bm{\theta}) := \frac{1}{n}\sum_{i=1}^{n}  f_{i}(\bm{\theta}) \text{ with } f_{i}(\bm{\theta}) := \frac{1}{m}\sum_{x\in{D}_i} l(\bm{\theta}; x),
% \end{align}
\begin{align}\label{fed_obj}
\min_{\bm{\theta}  \in \mathbb{R}^d} f(\bm{\theta}) := \frac{1}{n}\sum_{i\in[n]} f_{i}(\bm{\theta}),
\end{align}
where $f_{i}(\bm{\theta}) = \mathbb{E}_{z \in  D_i}[l(\bm{\theta};z)]$ is the local loss function of client $i$, $z$ represents a datapoint sampled from ${D}_i$, and $l(\bm{\theta};z)$ denotes the loss of model $\bm{\theta}$ on datapoint $z$. For $i\neq j$, the data distributions of ${D}_i$ and ${D}_j$ may be different. 
% (Cannot use stochastic form for DP papers due to the definition of adjacent datasets.)

% where $f_i$ represents the local objective function of device $i$ and is defined as the average loss of the global model over user $i$'s dataset.
%------------------------------------------------%
\begin{figure}[t]
\vspace*{-10pt}
\centering
\includegraphics[width=3in]{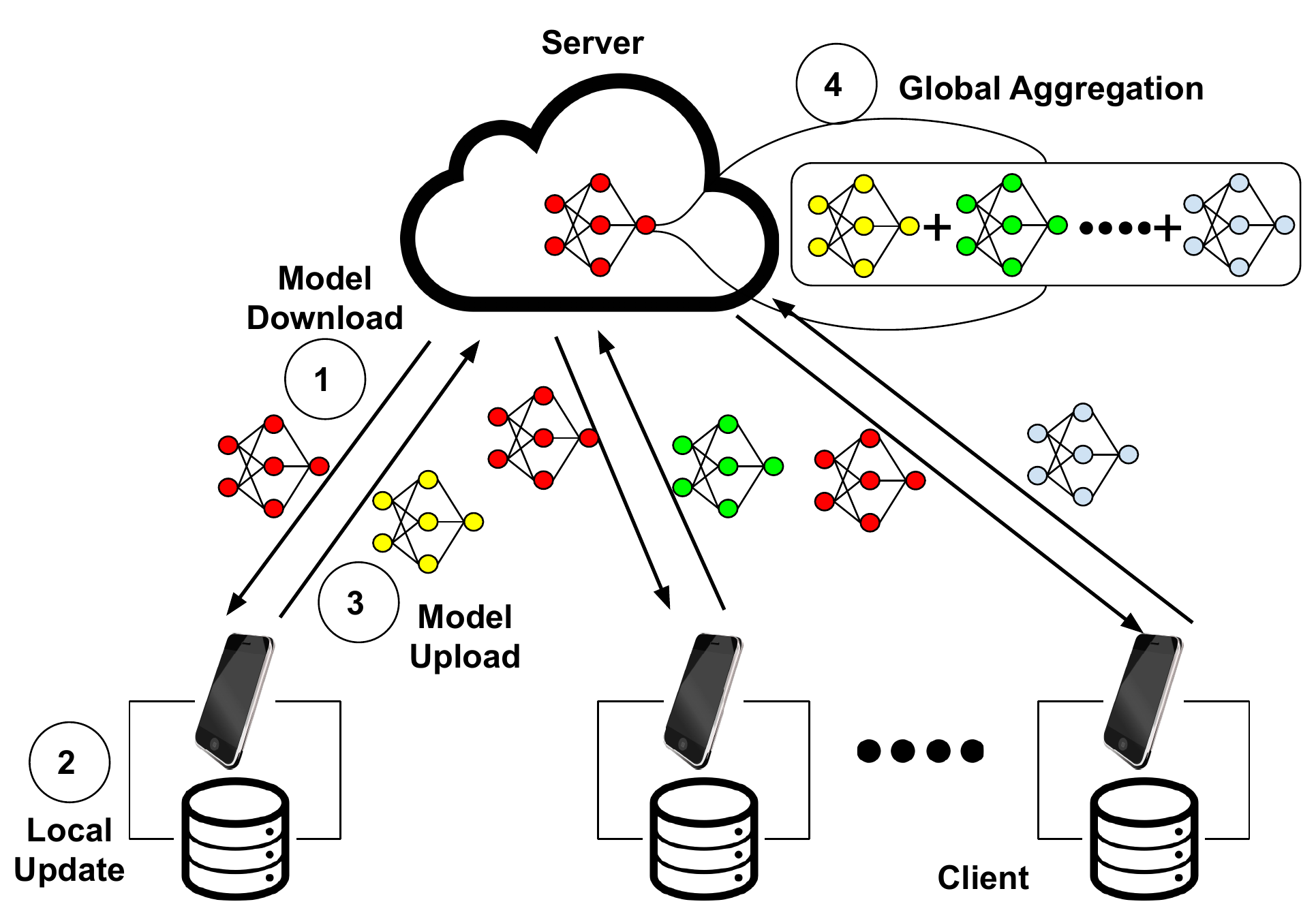}
\caption{An exemplary FL system. %\hl{(Figure needs to be changed. No source file, need to redrawn by yourself.)}
}
\label{fig:fl_system}
% \vspace*{-0.2in}

\end{figure}
%------------------------------------------------%
%%%%%%%%%%%%%%%%%%%%%%%%%%%%%%%%%%%%%%%%%%%%%%%%%%%%
\subsection{Attack Model}
%%%%%%%%%%%%%%%%%%%%%%%%%%%%%%%%%%%%%%%%%%%%%%%%%%%%
The adversary can be the ``honest-but-curious'' server or clients in the system. The adversary will honestly follow the designed training protocol but is curious about a target client’s private data and wants to infer it from the shared messages. Furthermore, some clients can collude with the server or each other to infer private information about a specific victim client. Besides, the adversary could also be the passive outside attacker. These attackers can eavesdrop on all the shared messages during the execution of the training but will not actively inject false messages into or interrupt message transmissions.

%%%%%%%%%%%%%%%%%%%%%%%%%%%%%%%%%%%%%%%%%%%%%%%%%%%%
\subsection{Achieving Client-level DP in FL}\label{subsec:fedavg}
%%%%%%%%%%%%%%%%%%%%%%%%%%%%%%%%%%%%%%%%%%%%%%%%%%%%

As the classic and most widely-used algorithm in the FL setting, Federated Averaging (FedAvg) \cite{mcmahan2017communication} solves \eqref{fed_obj} by running multiple iterations of stochastic gradient descent (SGD) in parallel on a subset of clients and then averaging the resulting local model updates at a central server periodically. Compared with distributed SGD, FedAvg is shown to achieve the same model accuracy with fewer communication rounds. Specifically, FedAvg involves $T$ communication rounds, and each round can be divided into four stages as shown in Fig.~\ref{fig:fl_system}. First, at the beginning of round $t \in \{0,\dots,T-1\}$, the server randomly selects a subset $\mathcal{W}^t$ of $r$ clients and sends them the latest global model $\bm{\theta}^{t}$ to perform local computations. Second, each client $i \in \mathcal{W}^t$ runs $\tau$ iterations of SGD on its local dataset to update its local model. Let $\bm{\theta}_i^{t,s}$ denote client $i$'s model after the $s$-th local iteration at the $t$-th communication round where $s \in [0, \tau-1]$. By model initialization, we have $\bm{\theta}_i^{t, 0} = \bm{\theta}^{t}$. Then the update rule at client $i$ is represented as
\begin{equation}\label{eqn:local-sgd}
\bm{\theta}_i^{t,s+1} = \bm{\theta}_i^{t,s} - \eta_l \bm{g}_i^{t,s}, \forall s = 0, \ldots, \tau - 1,
\end{equation}
where $\eta_l$ is the local learning rate and $\bm{g}_i^{t,s}:=(1/B)\sum_{z\in\xi_i^{t,s}}\nabla l(\bm{\theta}_i^{t,s},z)$ represents the stochastic gradient over a mini-batch $\xi_i^{t,s}$ of $B$ datapoints sampled from $D_i$. Third, the client $i$ uploads its model update $\Delta_i^t:=\bm{\theta}^{t}-\bm{\theta}_i^{t,\tau} $ to the server. Fourth, after receiving all the local model updates, the server updates the global model by
\begin{equation}\label{eqn:server_avg}
\bm{\theta}^{t+1} = \bm{\theta}^{t} - \frac{1}{r} \sum_{i\in\mathcal{W}^t} \Delta_i^t.
\end{equation}
The same procedure repeats at the next round $t + 1$ thereafter until satisfying certain convergence criteria. 

In FedAvg, clients repeatedly upload the local model updates of large dimension $d$ (e.g., millions of model parameters for DNNs) to the server and download the newly-updated global model from the server many times in order to learn an accurate global model. Since the bandwidth between the server and clients could be rather limited, especially for uplink transmissions, the overall communication cost could be very high. Furthermore, although FedAvg can avoid direct information leakage by keeping the local dataset on the client, the model updates shared at each round can still leak private information about the local dataset, as demonstrated in recent advanced attacks such as model inversion attack \cite{fredrikson2015model} and membership inference attack \cite{shokri2017membership}. %The above two drawbacks motivate us to develop FL schemes that are both privacy-preserving and communication-efficient. 

As the commonly-used privacy protection technique for machine learning, DP mechanisms have been used to mitigate the above-mentioned privacy leakage in FedAvg. According to Section~\ref{sec:pre}, the DP definitions apply to a range of different granularities, depending on how the adjacent datasets are defined. In this paper, we define the adjacent datasets by adding or removing the entire local data of a client and aim to protect whether one client participates in training or not in FL, resulting in the \emph{client-level DP} \cite{mcmahan2017learning}. In comparison, the commonly used privacy notion in standard non-federated learning DP is \emph{record-level DP} \cite{abadi2016deep}, where the adjacent datasets are defined by adding or removing a single training example of a client, and only the privacy of one training example is protected. Therefore, client-level DP is stronger than record-level DP and has been shown to be more relevant to the cross-device federated learning we considered, where there are a large number of participating devices, and each device can contribute multiple data records\cite{mcmahan2017learning}.

To provide client-level DP in FL without a fully trusted server, prior studies \cite{mcmahan2017learning,geyer2017differentially} have proposed the model perturbation mechanism to prevent the privacy leakage from model updates in FedAvg, known as DP-FedAvg. As the pseudo-code of DP-FedAvg given in Algorithm~\ref{algorithm-dpfedavg}, DP-FedAvg has the following changes to FedAvg at each FL round: 1) Each client's model update is clipped to have a bounded $\ell_2$-norm (line~\ref{dp-fedavg-clip}); 2) Gaussian noise is added to the clipped model update at each client (line~\ref{dp-fedavg-noise}); 3) Final local model updates are encrypted following a secure aggregation protocol (e.g., \cite{bonawitz2017practical,bonawitz2019towards}) and sent to the server for aggregation (line~\ref{dp-fedavg-encrypt}).

%---------------------------------------------------%
\begin{algorithm}[htpb]
\caption{DP-FedAvg}\label{algorithm-dpfedavg}
\begin{algorithmic}[1]
\REQUIRE number of selected clients per round $r$, number of training rounds $T$, local update period $\tau$, local learning rate $\eta_l$, clipping threshold $C$, noise multiplier $\sigma$. 
\Serverexe
\STATE Initialize $\bm{\theta}^0\in\mathbb{R}^d$
\FOR{$t=0$ to $T-1$}
    \STATE Sample a set of $r$ clients uniformly at random without replacement, denoted by  $\mathcal{W}^t \subseteq[n]$
    \STATE Broadcast $\bm{\theta}^t$ to all clients in $\mathcal{W}^t$
    \FOR{each clients $i \in \mathcal{W}^t$ \textbf{in parallel}}
        \STATE $\textbf{y}_i^t \leftarrow \textbf{ClientUpdate}(\bm{\theta}^t)$ 
    \ENDFOR
    \STATE $\bm{\theta}^{t+1} \leftarrow \bm{\theta}^t - (1/r) \sum_{i\in\mathcal{W}^t} \textbf{y}_i^{t}$ \label{dp-fedavg-decrypt}
    \ENDFOR
    \returnx $\bm{\theta}^{T}$
\item[]
\Clientexedpfed
    \STATE $\bm{\theta}_i^{t, 0} \gets \bm{\theta}^t$
    \FOR{$s=0$ to $\tau-1$}
        \STATE Compute a mini-batch stochastic gradient $\bm{g}_i^{t,s}$
        \STATE ${\bm{\theta}}_i^{t, s+1} \leftarrow \bm{\theta}_i^{t, s} - \eta_l   \bm{g}_i^{t,s}$
    \ENDFOR
    \STATE $\hat{\Delta}_i^t \leftarrow \bm{\theta}^{t} -{\bm{\theta}}_i^{t,\tau}$
     \STATE $\bar{\Delta}_i^t \leftarrow \hat{\Delta}_i^t \times \min(1, C/\|\hat{\Delta}_i^t\|_2)$ \label{dp-fedavg-clip}
    \STATE $\Delta_i^t \leftarrow \bar{\Delta}_i^t + \mathcal{N}(0, (C^2\sigma^2/r) \cdot \bm{I}_d) $ \label{dp-fedavg-noise}
    \STATE Encrypt $\Delta_i^t$ and send it to the server via secure aggregation \label{dp-fedavg-encrypt}
\end{algorithmic}
% \vspace*{-5pt}
\end{algorithm}
%---------------------------------------------------%

Note that the use of secure aggregation is a common practice in the literature to achieve client-level DP in FL, and the design of a new secure aggregation protocol is out of the scope of this paper. Secure aggregation enables the server to learn just an aggregate function of the clients' local model updates, typically the sum, and nothing else, so the Gaussian noise is added to prevent privacy leakage from the sum of local model updates. Specifically, in secure aggregation, clients generate randomly sampled zero-sum mask vectors locally by working in the space of integers modulo $m$ and sampling the elements of the mask uniformly from $\mathbb{Z}_m$. When the server computes the modular sum of all the masked updates, the masks cancel out, and the server obtains the exact sum of local model updates. As with the existing works in FL \cite{goryczka2013secure,truex2019hybrid,valovich2017computational}, we ignore the finite precision and modular summation arithmetic associated with secure aggregation in this paper, noting that one can follow the strategy in \cite{bonawitz2019federated} to transform the real-valued vectors into integers for minimizing the approximation error of recovering the sum.

Although we can improve privacy and achieve client-level DP in FL by adding Gaussian noise locally and using secure aggregation, the resulting accuracy of the trained model is often low due to the significant intensity of added Gaussian noise. Moreover, communication cost has never been considered in those studies. This motivates us to develop a new differentially private FL scheme that can maintain high model accuracy while reducing the communication cost.

\section{Fed-SMP: Federated Learning with Sparsified Model Perturbation}\label{subsec:Fed-SMP}
%%%%%%%%%%%%%%%%%%%%%%%%%%%%%%%%%%%%%%%%%%%%%%%%%%

% In this subsection, we develop a new FL scheme called Fed-SMP with the goal of providing client-level DP while improving communication efficiency at the same time. The \textbf{threat model} considered in this paper is defined as follows.  
% %
% %the server and clients are assumed to be  ``honest-but-curious'', that is, they will follow the designed protocol honestly but may be curious about a target client's private dataset and want to infer it from the shared messages during the collaboration process.
% %
% The adversary can be the “honest-but-curious” server or clients in the system. The adversary will honestly follow the designed training protocol but is curious about a target client’s private data and wants to infer it from the shared messages. Furthermore, some clients can collude with the server or each other to infer private information about a specific victim client. Besides, the adversary could also be the passive outside attacker. These attackers can eavesdrop all the shared messages in the execution of the training protocol but will not actively inject false messages into or interrupt message transmissions.

%Our method aims to provide DP guarantee to each client and ensure that sensitive information that leaves each client is bounded throughout the learning stages. 

%---------------------------------------------------%
\begin{algorithm*}[htpb]
\vspace*{-10pt}
\caption{Fed-SMP: Federated Learning with Sparsified Model Perturbation}\label{algorithm}
\begin{multicols}{2}
\begin{algorithmic}[1]
\REQUIRE number of selected clients per round $r$, number of training rounds $T$, local update period $\tau$, local learning rate $\eta_l$, clipping threshold $C$, noise multiplier $\sigma$, compression parameter $k$, sparsifier $\spar$ ($\randk_k$ or $\topk_k$). %encryption/decryption functions $\Enc(\cdot)$ and $ \Dec(\cdot)$ in the secure aggregation protocol.
\item[]
\Serverexe
\STATE Initialize $\bm{\theta}^0\in\mathbb{R}^d$
\FOR{$t=0$ to $T-1$}
    \STATE Sample a set of $r$ clients uniformly at random without replacement, denoted by $\mathcal{W}^t \subseteq[n]$ \label{alg1:ln_1}
    \IF{$\spar$ is $\randk_k$}
        \STATE $\bm{m}^t \leftarrow \textbf{SelectRandk}(\bm{\theta}^t,k)$
    \ELSIF{$\spar$ is $\topk_k$}
        \STATE $\bm{m}^t \leftarrow \textbf{SelectTopk}(\bm{\theta}^t,k)$
    \ENDIF
    % \STATE Select a random set of $k$ coordinates and create a corresponding mask vector $\bm{m}^t \in \{0, 1\}^d$
    \STATE Broadcast $\bm{\theta}^t$ and $\bm{m}^t$ to all clients in $\mathcal{W}^t$ \label{alg1:ln_2}
    \FOR{each clients $i \in \mathcal{W}^t$ \textbf{in parallel}}
        \STATE $\textbf{y}_i^t \leftarrow \textbf{ClientUpdate}(\bm{\theta}^t, \bm{m}^t, \spar)$ 
    \ENDFOR
%    \STATE Server updates $\bm{\theta}^{t+1} \leftarrow \bm{\theta}^t - (/r) \sum_{i\in\mathcal{W}^t}  {{\Delta}_i^{t}}$
    \STATE $\bm{\theta}^{t+1} \leftarrow \bm{\theta}^t - (1/r) \sum_{i\in\mathcal{W}^t} \textbf{y}_i^{t}$ \label{alg1:ln_aggre}
    \ENDFOR
    \returnx $\bm{\theta}^{T}$
\item[]
\Selectrandk
    \STATE Select a random set of $k$ coordinates of $\bm{\theta}^t$ and create a corresponding mask vector $\bm{m}^t \in \{0, 1\}^d$
    \returnx $\bm{m}^t$ \label{alg1:ln_randk}
\item[]
\Selecttopk
\REQUIRE public dataset ${D}_p$, update period $\tau_p$
    \STATE $\bm{\theta}_p^0 \leftarrow \bm{\theta}^t$ \label{alg1:ln_topk_1}
    \FOR{$s=0$ to $\tau-1$}
        \STATE Compute a mini-batch stochastic gradient $\bm{g}_p^{s}$ over ${D}_p$
        \STATE ${\bm{\theta}}_p^{s+1} \leftarrow \bm{\theta}_p^{s} - \eta_l \bm{g}_p^{s}$
    \ENDFOR
    \STATE $\Delta_p \leftarrow  \bm{\theta}^{t}- {\bm{\theta}}_p^{\tau}$
    \STATE Select the top $k$ coordinates of $\abs{\Delta_p}$ and create a corresponding mask vector $\bm{m}^t \in \{0, 1\}^d$
    \returnx $\bm{m}^t$ \label{alg1:ln_topk_2}
\item[]
\Clientexe
    \STATE $\bm{\theta}_i^{t, 0} \gets \bm{\theta}^t$ \label{alg1:ln_3}
    \FOR{$s=0$ to $\tau-1$}
        \STATE Compute a mini-batch stochastic gradient $\bm{g}_i^{t,s}$
        \STATE ${\bm{\theta}}_i^{t, s+1} \leftarrow \bm{\theta}_i^{t, s} - \eta_l   \bm{g}_i^{t,s}$
    \ENDFOR \label{alg1:ln_4}
    \STATE $\Delta_i^t \leftarrow \textbf{DP-spar}(\bm{\theta}^{t} -{\bm{\theta}}_i^{t,\tau}, \bm{m}^t, \spar)$
    % 
    % \STATE \hl{(encryption using secure aggregation protocol here? Check the compression-boosted paper to know how to write it correctly.)} \label{alg1:ln_7}
    \STATE Encrypt $\Delta_i^t$ and send it to the server via secure aggregation \label{alg1:ln_enc}
    % \returnx $\Enc(\Delta_i^t)$ \label{alg1:ln_enc}
\item[]
\DPspar
    \IF{$\spar$ is $\randk_k$}  \label{alg1:ln_dpspar_1}
        \STATE $\Delta^\prime \leftarrow \frac{d}{k} \times \Delta \odot \bm{m}^t$ \label{alg1:ln_scale}
    \ELSIF{$\spar$ is $\topk_k$}
        \STATE $\Delta^\prime \leftarrow \Delta \odot \bm{m}^t$
    \ENDIF \label{alg1:ln_dpspar_0}
    \STATE $\hat{\Delta} \leftarrow \Delta^\prime \times \min(1, C/\|\Delta^\prime\|_2)$ \label{alg1:ln_clip}
    \returnx $\hat{\Delta} + (\mathcal{N}(0, (C^2\sigma^2/r) \cdot \bm{I}_d) \odot \bm{m}^t) $ \label{alg1:ln_dpspar_2}
\end{algorithmic}
\end{multicols}
\vspace*{-5pt}
\end{algorithm*}
%---------------------------------------------------%
 
In this section, we develop a new FL scheme called Fed-SMP to provide client-level DP with high model accuracy while improving communication efficiency at the same time. To ensure easy integration into existing packages/systems, Fed-SMP follows the same overall procedure of FedAvg as depicted in Fig.~\ref{fig:fl_system} and employs a novel integration of Gaussian mechanism and sparsification in the local update stage. This guarantees that each client's shared local model update is both sparse and differentially private. The two specific sparsifiers considered in Fed-SMP are defined as follows: 

%---------------------------------------------------%
\begin{definition}[Random and Top-$k$ Sparsifiers]\label{def:sparsifier}
For a parameter $1\leq k\leq d$ and vector $\mathbf{x}\in\mathbb{R}^d$, the random sparsifier $\randk_k:\mathbb{R}^d \times \Omega_k \rightarrow \mathbb{R}^d$ and top-$k$ sparsifier $\topk_k:\mathbb{R}^d \rightarrow \mathbb{R}^d$  are defined as
\begin{align}
&[\randk_k(\mathbf{x},\omega)]_j:=
\begin{cases} 
[\mathbf{x}]_j, & \text{if } j\in\omega\\ 
0, & \text{otherwise}
\end{cases},
\\
&[\topk_k(\mathbf{x})]_j:=
\begin{cases} 
[\mathbf{x}]_{\pi(j)}, & \text{if } j\leq k\\ 
0, & \text{otherwise}
\end{cases},
\end{align}
where $\Omega_k = {[d] \choose k}$ denotes the set of all $k$-element subsets of $[d]$, $\omega$ is chosen uniformly at random, i.e., $\omega\sim_{u.a.r}\Omega_k$, and $\pi$ is a permutation of $[d]$ such that $|[\mathbf{x}]_{\pi(j)}| \geq |[\mathbf{x}]_{\pi(j+1)}|$ for $j\in[1,d-1]$.
\end{definition}

In Fed-SMP with both $\randk_k$ and $\topk_k$ sparsifier, only $k$ coordinates of the local model update will be sent out. Generally speaking, $\randk_k$ sparsifier may discard coordinates that are actually important, which inevitably degrades model accuracy when $k$ is small. On the other hand, $\topk_k$ sparsifier keeps the coordinates with the largest magnitude and can achieve higher model accuracy for small $k$. This seems to indicate that $\topk_k$ is always the better choice. However, the two sparsifiers have different privacy implications. For $\randk_k$ sparsifier, since the set of selected coordinates $\omega$ is chosen uniformly at random, the coordinates themselves are not data-dependent and thus do not contain any private information of client data. On the other hand, the set of selected coordinates for $\topk_k$ sparsifier depends on the values of model parameters and hence contains private information of client data and cannot be used like that of $\randk_k$ sparsifier. It is worth noting that although $\randk_k$ sparsifier is a randomized mechanism, it does not provide any privacy guarantee by itself in terms of DP and needs to be combined with Gaussian mechanism for rigorous DP guarantee.

Using the above sparsifiers for client-level DP in FL has new challenges. As mentioned before, secure aggregation is a key privacy-enhancing technique to achieve client-level DP in practical FL systems. However, the naive application of $\randk_k$ or $\topk_k$ sparsifier to the local model update of each client typically results in a different set of $k$ coordinates for each client, preventing us from only encrypting the $k$ selected coordinates of each client in the secure aggregation protocol. 
%
%which is not compatible with secure aggregation. 
%
%making it difficult to apply secure aggregation directly in $k$ coordinates. 
%
One can apply secure aggregation to all the $d$ coordinates, but the communication efficiency benefit of sparsification will get lost. In the following, we design new $\randk_k$ and $\topk_k$ sparsifiers in Fed-SMP that are compatible with secure aggregation. Specifically, we let the selected clients keep the same set of $k$ active coordinates at each round; therefore, those clients can transmit the sparsified model update directly using the secure aggregation protocol and save the communication cost. 
% (This sentence needs to be integrated here as a basic idea behind Fed-SMP..)}  

The pseudo-code for the proposed Fed-SMP is provided in Algorithm~\ref{algorithm}. At the beginning of round $t$, the server randomly selects a set $\mathcal{W}^t$ of $r$ clients and broadcasts to them the current global model $\bm{\theta}^t$ and mask vector $\bm{m}^t \in \{0,1\}^d$ (lines~\ref{alg1:ln_1}-\ref{alg1:ln_2}). The $j$-th coordinate of mask vector $\bm{m}^t$ equals to 1 if that coordinate is selected by the sparsifier at round $t$ and 0 otherwise. In Fed-SMP with $\randk_k$ sparsifier, the $k$ coordinates are selected randomly from $[d]$ by the server (i.e., the procedure $\textbf{SelectRandk}(\cdot)$). In Fed-SMP with $\topk_k$ sparsifier, we let the server choose a set of top $k$ coordinates for all clients using a small public dataset $D_p$ at each round, avoiding privacy leakage from the selected coordinates.
% As we mentioned before, the top $k$ coordinates are selected based on the values of each coordinate, which contain private information of the global model parameters.
The distribution of the public dataset $D_p$ is assumed to be similar to the overall dataset distribution of clients. This assumption is common in the FL literature\cite{papernot2018scalable,alon2019limits, wang2020differentially, sharma2022federated}, where the server is assumed to have a small public dataset for model validation that mimics the overall dataset distribution.
% We assume that the public dataset $D_p$ mimics the overall data distribution of all clients' datasets $\{D_i\}_{i\in[n]}$. 
Specifically, at round $t$, the server first performs multiple iterations of SGD similar to \eqref{eqn:local-sgd} on the global model $\bm{\theta}^t$ using the public dataset $D_p$ and obtains the model difference $\Delta_p$. Then, the server selects a set of $k$ coordinates with the largest absolute values in $\Delta_p$ and generates a corresponding mask vector $\bm{m}^t$ (i.e., the procedure $\textbf{SelectTopk}(\cdot)$). 

After receiving the global model $\bm{\theta}^t$ and mask vector $\bm{m}^t$ from the server, each client $i \in \mathcal{W}^t$ initializes its local model to $\bm{\theta}^t$ and runs $\tau$ iterations of SGD to update its local model in parallel (lines~\ref{alg1:ln_3}-\ref{alg1:ln_4}). Then, the client $i$ sparsifies its local model update $\bm{\theta}^t - \bm{\theta}_i^{t, \tau}$ using the mask vector $\bm{m}^t$ (lines~\ref{alg1:ln_dpspar_1}-\ref{alg1:ln_dpspar_0}). The operator $\odot$ in Algorithm~\ref{algorithm} represents the element-wise multiplication. Note that for $\randk_k$ sparsifier, the model update will be scaled by $d/k$ to ensure an unbiased estimation on the sparsified model update (line~\ref{alg1:ln_scale}). Since there is no a priori bound on the size of the sparsified model update $\Delta^\prime$, each client will clip its sparsified model update in $\ell_2$-norm with clipping threshold $C$ so that $\|\hat{\Delta}\|_2 \leq  C$ (line~\ref{alg1:ln_clip}). 
%Note that the clipping of this form is a popular in deep networks for non-privacy reasons, though in that setting it usually suffices to clip the gradient. 
Next, each client perturbs its sparsified model update by adding independent Gaussian noise $\mathcal{N}(0, C^2\sigma^2/r)$ on the $k$ selected coordinates (line~\ref{alg1:ln_dpspar_2}), where $\sigma$ represents the noise multiplier. The noisy sparsified model update $\Delta_i^t$ is then encrypted as an input into a secure aggregation protocol and sent to the server (line~\ref{alg1:ln_enc}). Finally, the server computes the modular sum of all encrypted models to obtain the exact sum of local model updates (i.e.,  $\sum_{i\in \mathcal{W}^t}\textbf{y}_i^t=\sum_{i\in \mathcal{W}^t}\Delta_i^t$) and updates the global model for the next round (line~\ref{alg1:ln_aggre}). 
\section{Main Theoretical Results}\label{sec:main_results}
%%%%%%%%%%%%%%%%%%%%%%%%%%%%%%%%%%%%%%%%%%%%%%%%%%%%%%%%%%

In this section, we analyze the end-to-end privacy guarantee and convergence results of Fed-SMP. For better readability, we state the main theorems and only give the proof sketches in this section while leaving the complete proofs in the appendix. Before presenting our theoretical results, we make the following assumptions:

%---------------------------------%
\begin{assumption}[Smoothness]\label{assp:smooth}
Each local objective function $f_i:\mathbb{R}^d \rightarrow \mathbb{R}$ is $L$-smooth, i.e., for any $\mathbf{x}, \mathbf{y}\in \mathbb{R}^d$, 
\begin{equation*}
\|\nabla f_i(\mathbf{y})-\nabla f_i(\mathbf{x})\|\leq L \|\mathbf{y}-\mathbf{x}\|, \forall i \in [n]. 
\end{equation*}
\end{assumption}
%---------------------------------%
%---------------------------------%
\begin{assumption}[Unbiased Gradient and Bounded Variance]%[\hl{Same with Assumption 2 in }\cite{koloskova2019decentralized}]
\label{assp:bounded_variance}
% \hl{Rui: (a) check if the following revision is accurate. (b) No bound on global variance?} 
%Let $\bm{g}_i(\cdot)$ be the stochastic gradient over a mini-batch sampled from the distribution $\mathcal{D}_i$ at the $i$-th client. 
The stochastic gradient at each client is an unbiased estimator of the local gradient: $\mathbb{E}[\bm{g}_i(\mathbf{x})] = \nabla f_i(\mathbf{x})$, and has bounded variance: $\mathbb{E}[\|\bm{g}_i(\mathbf{x}) - \nabla f_i(\mathbf{x})\|^2] \leq \zeta_i^2, \forall \mathbf{x} \in \mathbb{R}^d, i \in [n]$, where the expectation is over all the local mini-batches. We also denote $\bar{\zeta}^2 := ({1}/{n})\sum_{i=1}^{n}\zeta_i^2$ for convenience. 
%
%The variance of the local gradient estimate is bounded, i.e., $\mathbb{E}_{z \sim \mathcal{D}_i}\|\nabla l_i(\mathbf{x}, z) - \nabla f_i(\mathbf{x})\|^2\leq \zeta_i^2$ for all $i\in[n]$.
%The function $f_i$ has a bounded local variance, i.e.,  $\mathbb{E}\|\bm{g}_i - \nabla f_i(\mathbf{x}) \|^2 \leq \zeta_{i}^2$ for all $\mathbf{x}\in\mathbb{R}^d$ and $i\in[n]$. 
\end{assumption}
%---------------------------------%
%---------------------------------%
\begin{assumption}[Bounded Dissimilarity]%[\hl{Same with Assumption 2 in }\cite{koloskova2019decentralized}]
\label{assp:bounded_divergence}
There exist constants $\beta^2\geq 1, \kappa^2\geq 0$ such that $(1/n)\sum_{i=1}^{n} \|\nabla f_i(\mathbf{x})\|^2 \leq \beta^2 \|(1/n)\sum_{i=1}^{n}\nabla f_i(\mathbf{x})\|^2 + \kappa^2$. If local objective functions are identical to each other, then we have $\beta^2=1$ and $\kappa^2=0$. 
\end{assumption}
%---------------------------------%
% \begin{assumption}
% \label{assp:bounded_gradient_coord}
% {Assume that each coordinate of the model difference} $\Delta_i^t$ is Algorithm~\ref{algorithm-1} is bounded, i.e., $|[\Delta_i^t]_j|\leq \sqrt{C/d}, \forall j\in[d], i\in[n],t\in[0,T-1]$, and hence the sensitivity of $[\Delta_i^t]_j$ is $\psi=2\sqrt{C/d}$.
% \end{assumption}
%---------------------------------%
Assumptions~\ref{assp:smooth} and \ref{assp:bounded_variance} are standard in the analysis of SGD \cite{bottou2018optimization}, and Assumption~\ref{assp:bounded_divergence} is commonly used in the federated optimization literature \cite{ward2019adagrad,li2019convergence} to capture the dissimilarities of local objectives under non-IID data distribution. 

% Assumption~\ref{assp:bounded_gradient_coord} can be enforced by the commonly used clipping technique in machine learning.

%%%%%%%%%%%%%%%%%%%%%%%%%%%%%%%%%%%%%%%%%%%%%%%%
\subsection{Privacy Analysis}
%%%%%%%%%%%%%%%%%%%%%%%%%%%%%%%%%%%%%%%%%%%%%%%

In this subsection, we provide the end-to-end privacy analysis of Fed-SMP based on RDP. 
% We first give the RDP guarantee in each round, and then derive the RDP guarantee after $T$ communication rounds and convert it to the traditional $(\epsilon, \delta)$-DP guarantee. 
Given the fact that the server only knows the sum of local model updates $\sum_{i\in\mathcal{W}^t}\Delta_i^t$ due to the use of secure aggregation, we need to compute the privacy loss incurred from releasing the sum of local model updates. Assume the client sets $\mathcal{W}$ and $\mathcal{W}^\prime$ differ in one client index $c$ such that $\mathcal{W}^\prime := \mathcal{W} \bigcup \{c\}$. For any adjacent datasets $D:=\{D_i\}_{i\in\mathcal{W}}$ and $D^\prime:=\{D_j\}_{j\in\mathcal{W}^\prime} = \{D_i\}_{i\in\mathcal{W}} \bigcup D_c$, according to Definition~\ref{def:sensitivity}, the $\ell_2$-sensitivity of the sum of local model updates is %\hl{(The notations below are in conflict with the algorithm description. They should be the vectors before noise addition.)}
\begin{equation*}
\psi_\Delta:= \sup_{D, D^\prime}\left\|\sum_{i\in\mathcal{W}}  {{\Delta}_i^{t}}(D_i) - \sum_{j\in\mathcal{W}^\prime}  {{\Delta}_j^{t}}(D_j)\right\|_2.
\end{equation*}
Due to the clipping, we have $\|{{\Delta}_i^{t}}(D_i)\|_2 \leq C, \forall i \in [n]$, and therefore $\psi_\Delta = \sup_{D, D^\prime}\left\| {{\Delta}_{c}^{t}}(D_{c})\right\|_2 \leq  C$. As the sum of Gaussian random variables is still a Gaussian random variable, the variance of the Gaussian noise added to each selected coordinate of the sum of local model updates is $C^2\sigma^2$. According to Lemmas~\ref{lemma:rdp_dp}--\ref{lemma:rdp_comr}, we compute the overall privacy guarantee of Fed-SMP as follows:
% %------------------------------------------------------%
% \begin{theorem}[Per-Round Privacy Guarantee]\label{thm:priv}
% Each communication round of Fed-SMP satisfies $(\alpha,\alpha /2\sigma^2)$-RDP. \hl{(Need to consider sampling or not here????)}% where $\rho(\alpha) = \alpha /2\sigma^2$. 
% \end{theorem}
% % \begin{proof}
% % By lemma~\ref{lemma:gaussian_mechanism}, after adding noise $\mathcal{N}(0,r\sigma^2)$, the sum of local model updates satisfies $(\alpha,2\alpha C^2/r\sigma^2)$-RDP
% % \end{proof}

% Then we account the overall privacy guarantee of Fed-SMP using Lemma~\ref{lemma:rdp_sub} as follows:
%-------------------------%
\begin{theorem}[Privacy Guarantee of Fed-SMP]
\label{thm:privacy_loss}
Suppose the client is sampled without replacement with probability $q:=r/n$ at each round. For any $\epsilon < 2\log(1/\delta)$ and $\delta\in(0,1)$, Fed-SMP satisfies $(\epsilon, \delta)$-DP after $T$ communication rounds if
\begin{equation*}
\sigma^2 \geq \frac{7q^2T(\epsilon + 2\log(1/\delta))}{\epsilon^2}.
\end{equation*}
% if $\sigma^2 \geq 0.7$ and $\alpha \leq (2/3)C^2\sigma^2\log(1/q\alpha(1+ \sigma^2)) + 1$,
\end{theorem}
\begin{proof}
After adding noise $\mathcal{N}(0,C^2\sigma^2)$ to the $k$ selected coordinates, the sum of local model updates satisfies $(\alpha, \alpha /2\sigma^2)$-RDP for each client in $\mathcal{W}^t$ at round $t$ by Lemma~\ref{lemma:gaussian_mechanism}. As $r$ clients are uniformly sampled from all clients at each round, the per-round privacy guarantee of Fed-SMP can be further amplified according to the subsampling amplification property of RDP in Lemma~\ref{lemma:rdp_sub}. The overall privacy guarantee of Fed-SMP follows by using the composition property in Lemma~\ref{lemma:rdp_comr} to compute the RDP guarantee after $T$ rounds and Lemma~\ref{lemma:rdp_dp} to convert RDP to $(\epsilon, \delta)$-DP. The details are given in Appendix~\ref{subsec:proof_priv}. 
\end{proof}

%%%%%%%%%%%%%%%%%%%%%%%%%%%%%%%%%%%%%%%%%%%%%%%%
\subsection{Convergence Analysis}\label{subsec:convergence}
%%%%%%%%%%%%%%%%%%%%%%%%%%%%%%%%%%%%%%%%%%%%%%%%
% why only provide the convergence result of randk
In this subsection, we provide the convergence result of Fed-SMP under the general non-convex setting in Theorem~\ref{thm:convergence}. %Note that the convergence of compressed federated learning algorithm like Fed-SMP requires that the output of the compression operator is an unbisased estimator of its input. It remains a big challenge to ensure convergence of biased sparsifier such as $\topk_k$ and hyper parameters need to be carefully selected to avoid divergence in practice.

%---------------------------------------------------%
\begin{theorem}[Convergence result of Fed-SMP]\label{thm:convergence}
Under Assumptions~\ref{assp:smooth}--\ref{assp:bounded_divergence}, assume the local learning rate satisfies $\eta_l\leq \min\{ (1/{24\tau L (\phi_k +1)\beta^2}, {1}/{4\tau L \sqrt{4\beta^2+2}},  1/12\tau L \}$ and $\|\hat{\Delta}\|_2 \leq C$, then the sequence of outputs $\bm{\theta}^t$ generated by Algorithm~\ref{algorithm} satisfies:
\begin{multline}\label{eqn:convergen_result}
\frac{1}{T}\sum_{t=0}^{T-1}\left\|\nabla f(\bm{\theta}^{t})\right\|^2 \leq \frac{8e_0}{T\eta_l\tau}   + {8}{\eta_l\tau L}(3\kappa^2 + 2\bar{\zeta}^2)\\
+  \left({8}\eta_l\tau L(2\kappa^2 + \bar{\zeta}^2) + \zeta^\prime \right)\phi_k   + \frac{4LkC^2\sigma^2}{\eta_l\tau r^2},
\end{multline}
where $e_0:= f({\bm{\theta}}^{0})-f^*$, $\phi_k:=1 - k/d$ and $\zeta^\prime:= {4\bar{\zeta}^2}/{\gamma}$ for $\topk_k$ sparsifier, $\phi_k:=d/k-1 $ and  $\zeta^\prime:=0$ for $\randk_k$ sparsifier, and $f^*$ represents the optimal objective value.

% Let Assumptions~\ref{assp:smooth}--\ref{assp:bounded_divergence} hold, and $L, \bar{\zeta}, \beta, \kappa$ be as defined therein. Suppose $\eta_l$ satisfies $\eta_l\leq \min\{ (1/{24\tau L (\phi_k +1)\beta^2}, {1}/{16\tau L \sqrt{8\beta^2+2}},  /24\tau L \}$.  Then after $T$ rounds of training, the iterates of Algorithm~\ref{algorithm-rand} satisfy:
% \begin{multline}
% \label{eqn:convergen_result}
%     \frac{1}{T}\sum_{t=0}^{T-1}\left\|\nabla f(\bm{\theta}^{t})\right\|^2 \leq \frac{8e_0}{T\eta_l\tau}   + {8}{\eta_l\tau L}(3\kappa^2 + 2\bar{\zeta}^2) \\
%     +  \left({8}\eta_l\tau L(2\kappa^2 + \bar{\zeta}^2) + \zeta^\prime \right)\phi_k   + \frac{4LkC^2\sigma^2}{\eta_l\tau r^2},
% \end{multline}
% where $e_0:=(f({\bm{\theta}}^{0})-f^*)$. For Fed-SMP with $\topk_k$, $\phi_k:=1-k/d $ and $\zeta^\prime:= {4\bar{\zeta}^2}/{\gamma}$; for Fed-SMP with $\randk_k$,  $\phi_k:=d/k-1 $ and  $\zeta^\prime:=0$.
\end{theorem}
\begin{proof}
The proof is given in Appendix~\ref{subsec:proof_conve_randk}-\ref{subsec:proof_conve_topk}.
\end{proof}
%---------------------------------------------------%
The convergence bound \eqref{eqn:convergen_result} contains three parts. The first two terms $ {8e_0}/{T\eta_l\tau} + {8}{\eta_l\tau L}(3\kappa^2 + 2\bar{\zeta}^2) $ represent the optimization error bound in FedAvg. The third term $({8}\eta_l\tau L(2\kappa^2 + \bar{\zeta}^2) + \zeta^\prime )\phi_k$ is the \emph{compression error} resulted from applying sparsification on local model updates. The last term ${4LkC^2\sigma^2}/{\eta_l\tau r^2}$ is the \emph{privacy error} resulted from adding DP noise to perturb local model updates. Both compression error and privacy error increase the error floor at convergence. When no sparsification is applied (i.e., $k = d$ and $\phi_k = 0$), the compression error is zero. When no DP noise is added (i.e., $\sigma = 0$), the privacy error is zero. The above result shows an explicit tradeoff between compression error and privacy error in Fed-SMP. As $k$ decreases, the variance of sparsification $\phi_k$ gets larger which leads to a larger compression error, but the privacy error decreases. Therefore, there exists an optimal parameter $k$ that can balance those two errors to minimize the convergence bound. 

\begin{table}[htpb]
\vspace*{-5pt}
\renewcommand{\arraystretch}{1.2}
\caption{Summary of results on Fashion-MNIST dataset.}
\label{tab:fmnist}
\begin{center}
\resizebox{\columnwidth}{!}{\begin{tabular}{c|l|c|c|c}
\hline
\multirow{1}{*}{Compression} & \multirow{2}{*}{Algorithm} & \multicolumn{3}{c}{Performance}  \\\cline{3-5}  
\multirow{1}{*}{ratio}& & {Accuracy ($\%$)} & {Cost (MB)} & { Privacy ($\epsilon$)} \\\cline{1-5} 
% &  & Mean & Std & Mean & Std \\\cline{3-8} 
\multirow{2}{*}{$p = 0.001$} 
& Fed-SMP-$\topk_k$ & $77.18 \pm 0.48$   & 0.02  & 1.01 \\\cline{2-5} 
% & FedAvg-$\topk_k$ & $81.48\pm 0.48$ & 0.02  & - \\\cline{2-5}
& Fed-SMP-$\randk_k$ & $27.16 \pm 10.20$  & 0.02  & 1.01 \\\hline 
% & FedAvg-$\randk_k$ & $30.57 \pm 1.31$  &0.03  & -\\\hline
\multirow{2}{*}{$p = 0.005$} 
& Fed-SMP-$\topk_k$ & $\mathbf{80.76 \pm 0.27}$ & 0.10  & 1.01 \\\cline{2-5} 
% & FedAvg-$\topk_k$ & $84.35\pm 0.50$  & 0.10  & - \\\cline{2-5}
& Fed-SMP-$\randk_k$ & $56.04 \pm 5.22$  & 0.10  & 1.01 \\\hline
% & FedAvg-$\randk_k$ & $61.95 \pm 0.97$   & 0.16 & - \\\hline

\multirow{2}{*}{$p = 0.01$} 
& Fed-SMP-$\topk_k$ &$80.76 \pm 0.47$ & 0.20  & 1.01 \\\cline{2-5} 
% & FedAvg-$\topk_k$ & $85.52\pm 0.32$ & 0.20  & - \\\cline{2-5} 
& Fed-SMP-$\randk_k$ & $65.61 \pm 1.04$  & 0.20  & 1.01 \\\hline
% & FedAvg-$\randk_k$ & $65.98\pm 2.00$  & 0.33 &- \\\hline

\multirow{2}{*}{$p = 0.1$} 
& Fed-SMP-$\topk_k$ &  $80.49 \pm 0.23$ & 2.00  & 1.01 \\\cline{2-5} 
% & FedAvg-$\topk_k$ & $86.68\pm 0.10$ & 2.00 & - \\\cline{2-5} 
& Fed-SMP-$\randk_k$ & $77.59 \pm 0.53$ &2.00  & 1.01 \\\hline 
% & FedAvg-$\randk_k$ & $78.89\pm 0.32$  &3.29  & - \\\hline

\multirow{2}{*}{$p = 0.2$} 
& Fed-SMP-$\topk_k$ & $80.16 \pm 0.09$ & 3.99 & 1.01 \\\cline{2-5} 
% & FedAvg-$\topk_k$ & $86.85\pm 0.10$ & 3.99 & - \\\cline{2-5} 
& Fed-SMP-$\randk_k$ &$79.12 \pm 0.40$  & 3.99  & 1.01 \\\hline
% & FedAvg-$\randk_k$ & $ 82.51\pm 0.07$  & 6.57  & - \\\hline 

\multirow{2}{*}{$p = 0.4$}  
& Fed-SMP-$\topk_k$ & $80.26 \pm 0.25$ & 7.98 & 1.01 \\\cline{2-5} 
% & FedAvg-$\topk_k$ & $87.04\pm 0.37$ & 7.98  & - \\\cline{2-5} 
& Fed-SMP-$\randk_k$ & $\mathbf{79.88 \pm 0.37}$  & 7.98 & 1.01 \\\hline
% & FedAvg-$\randk_k$ & $84.89\pm 0.19$  & 13.14  & - \\\hline

\multirow{2}{*}{$p = 0.8$}  
& Fed-SMP-$\topk_k$ & $79.77 \pm 1.00$ &15.97 & 1.01 \\\cline{2-5} 
% & FedAvg-$\topk_k$ & $86.85\pm 0.30$  &15.97  & - \\\cline{2-5} 
& Fed-SMP-$\randk_k$ & $79.82 \pm 0.35$ & 15.97  & 1.01 \\\hline
% & FedAvg-$\randk_k$ & $86.40\pm 0.06$  & 26.28 & - \\\hline

\multirow{2}{*}{$ p = 1.0$} 
& DP-FedAvg &$72.72 \pm 3.72$ & 19.96 & 1.01 \\\cline{2-5}
& FedAvg & $86.98 \pm 0.12$  & 19.96 & - \\
\hline
\end{tabular}}
\end{center}
\vspace*{-10pt} 
\end{table}
%---------------------%
%---------------------%
\begin{table}[htpb]
% \vspace*{-5pt}
\renewcommand{\arraystretch}{1.2}
\caption{Summary of results on SVHN dataset.}\label{tab:svhn}
\begin{center}
\resizebox{\columnwidth}{!}{\begin{tabular}{c|l|c|c|c}
\hline
\multirow{1}{*}{Compression} & \multirow{2}{*}{Algorithm} & \multicolumn{3}{c}{Performance}  \\\cline{3-5} 
 \multirow{1}{*}{ratio} & & {Accuracy ($\%$)} & {Cost (MB)} & { Privacy ($\epsilon$)} \\\cline{1-5} 
% &  & Mean & Std & Mean & Std \\\cline{3-8} 

\multirow{2}{*}{$p = 0.005$} 
& Fed-SMP-$\topk_k$ & $80.94 \pm 0.85$   & 0.21  & 1.01 \\\cline{2-5} 
% & FedAvg-$\topk_k$ & $81.48\pm 0.48$ & 0.02  & - \\\cline{2-5}
& Fed-SMP-$\randk_k$ & $19.45 \pm 0.42$  & 0.21  & 1.01 \\\hline 
% & FedAvg-$\randk_k$ & $30.57 \pm 1.31$  &0.03  & -\\\hline
\multirow{2}{*}{$p = 0.01$} 
& Fed-SMP-$\topk_k$ & $\mathbf{81.12 \pm 1.08}$  & 0.41  & 1.01 \\\cline{2-5} 
% & FedAvg-$\topk_k$ & $84.35\pm 0.50$  & 0.10  & - \\\cline{2-5}
& Fed-SMP-$\randk_k$ & $19.75 \pm 0.12$  & 0.41  & 1.01 \\\hline
% & FedAvg-$\randk_k$ & $61.95 \pm 0.97$   & 0.16 & - \\\hline

\multirow{2}{*}{$p = 0.05$} 
& Fed-SMP-$\topk_k$ &$80.22 \pm 1.02$ & 2.06  & 1.01 \\\cline{2-5} 
% & FedAvg-$\topk_k$ & $85.52\pm 0.32$ & 0.20  & - \\\cline{2-5} 
& Fed-SMP-$\randk_k$ & $28.15 \pm 2.46$  & 2.06  & 1.01 \\\hline
% & FedAvg-$\randk_k$ & $65.98\pm 2.00$  & 0.33 &- \\\hline

\multirow{2}{*}{$p = 0.1$} 
& Fed-SMP-$\topk_k$ &  $79.64 \pm 0.29$ & 4.12  & 1.01 \\\cline{2-5} 
% & FedAvg-$\topk_k$ & $86.68\pm 0.10$ & 2.00 & - \\\cline{2-5} 
& Fed-SMP-$\randk_k$ & $64.50 \pm 2.95$  & 4.12  & 1.01 \\\hline 
% & FedAvg-$\randk_k$ & $78.89\pm 0.32$  &3.29  & - \\\hline

\multirow{2}{*}{$p = 0.2$} 
& Fed-SMP-$\topk_k$ & $77.91 \pm 0.37$ & 8.24 & 1.01 \\\cline{2-5} 
% & FedAvg-$\topk_k$ & $86.85\pm 0.10$ & 3.99 & - \\\cline{2-5} 
& Fed-SMP-$\randk_k$ & $76.94 \pm 0.58$ & 8.24  & 1.01 \\\hline
% & FedAvg-$\randk_k$ & $ 82.51\pm 0.07$  & 6.57  & - \\\hline 

\multirow{2}{*}{$p = 0.4$}  
& Fed-SMP-$\topk_k$ & $75.58 \pm 0.40$ & 16.47 & 1.01 \\\cline{2-5} 
% & FedAvg-$\topk_k$ & $87.04\pm 0.37$ & 7.98  & - \\\cline{2-5} 
& Fed-SMP-$\randk_k$ &  $\mathbf{77.07 \pm 0.48}$ & 16.47 & 1.01 \\\hline
% & FedAvg-$\randk_k$ & $84.89\pm 0.19$  & 13.14  & - \\\hline

\multirow{2}{*}{$p = 0.8$}  
& Fed-SMP-$\topk_k$ & $73.18 \pm 0.69$ & 32.94 & 1.01 \\\cline{2-5} 
% & FedAvg-$\topk_k$ & $86.85\pm 0.30$  &15.97  & - \\\cline{2-5} 
& Fed-SMP-$\randk_k$ & $73.90 \pm 0.37$ & 32.94  & 1.01 \\\hline
% & FedAvg-$\randk_k$ & $86.40\pm 0.06$  & 26.28 & - \\\hline

\multirow{2}{*}{$p = 1.0$} 
& DP-FedAvg &$72.46 \pm 0.54$ & 41.18 & 1.01 \\\cline{2-5}
& FedAvg & $88.47 \pm 0.17$ & 41.18 & - \\
\hline
\end{tabular}}
\end{center}
\vspace*{-10pt}
\end{table}
%---------------------%

\section{Performance Evaluation}\label{sec:exp}
In this section, we evaluate the performance of Fed-SMP with $\randk_k$ and $\topk_k$ sparsifiers (denoted by Fed-SMP-$\randk_k$ and Fed-SMP-$\topk_k$, respectively) by comparing it with the following baselines:
\begin{itemize}
\item FedAvg: the classic FL algorithm served as the baseline without privacy consideration;
\item FedAvg-$\topk_k$: a communication-efficient variant of FedAvg, where the local model update of each client is compressed using $\topk_k$ sparsifier before being uploaded;
\item FedAvg-$\randk_k$: another communication-efficient variant of FedAvg, where the local model update of each client is compressed using $\randk_k$ sparsifier before being uploaded;
\item DP-FedAvg \cite{mcmahan2017learning}: the state-of-art differentially-private variant of FedAvg that achieves client-level DP, where the full-precision local model update from each client is clipped with clipping threshold $C$ and then perturbed by adding Gaussian noise drawn from the distribution $\mathcal{N}(0,(C^2\sigma^2/r) \cdot \mathbf{I}_d)$, where $\sigma$ is the noise multiplier and $r$ is the number of selected clients per round.
\end{itemize}
% We account the overall privacy loss of DP-FedAvg and Fed-SMP in terms of RDP using Lemma~\ref{lemma:gaussian_mechanism}-\ref{lemma:rdp_comr} and then convert it to $(\epsilon,\delta)$-DP using Lemma~\ref{lemma:rdp_dp}. 
For fair comparisons, all the algorithms use the same secure aggregation protocol. When $p=1.0$, Fed-SMP-$\topk_k$ and Fed-SMP-$\randk_k$ reduce to DP-FedAvg, and FedAvg-$\topk_k$ and FedAvg-$\randk_k$ reduce to FedAvg.

\subsection{Dataset and Model}
% We conduct our experiments on Fashion-MNIST\cite{xiao2017fashion} dataset and extended SVHN dataset\cite{netzer2011reading}, two common benchmarks for differentially private machine learning. Note that while these datasets are considered as ``solved'' in the computer vision community, achieving high accuracy with strong privacy guarantee remains difficult on these datasets\cite{papernot2020tempered}.

We use Fashion-MNIST\cite{xiao2017fashion} dataset, SVHN dataset\cite{netzer2011reading} and Shakespeare dataset\cite{mcmahan2017communication}, three common benchmarks for differentially private machine/federated learning. The first two are image classification datasets, and the last one is a text dataset for next-character-prediction. Note that while Fashion-MNIST and SVHN datasets are considered as ``solved'' in the computer vision community, achieving high accuracy with strong privacy guarantee remains difficult on these datasets\cite{papernot2020tempered}.

The Fashion-MNIST dataset consists of a training set of 60,000 examples and a test set of 10,000 examples. Each example is a $28 \times 28$ grayscale image associated with a label from 10 classes. For Fed-SMP-$\topk_k$, we randomly sample 1,000 examples from the training set as the public dataset $D_p$ and evenly partition the remaining 59,000 examples across 6,000 clients. For other algorithms, we evenly split the training data across 6,000 clients. %in an IID manner.
We use the convolutional neural network (CNN) model in \cite{mcmahan2017communication} on Fashion-MNIST that consists of two $5\times5$ convolution layers (the first with 32 filters, the second with 64 filters, each followed with $2\times2$ max pooling and ReLu activation), a fully connected layer with 512 units and ReLu activation, and a final softmax output later, which results in about $1.6$ million total parameters.

% %---------------------%
% \begin{figure}[htp]
% % \vspace*{-10pt}
% \centering
% \subfloat[Impact of noise magnitude $z$.]{\includegraphics[width=0.49\textwidth]{impact_sigma_fminst_svhn.pdf}\label{fig:sigma}}\\
% \subfloat[Impact of clipping threshold $C$.]{\includegraphics[width=0.49\textwidth]{impact_clip_fmnist_svhn.pdf}\label{fig:c}}
% % \vspace*{-10pt}
% \caption{(a) Testing accuracy of DP-FedAvg under different noise magnitudes; (b) Testing accuracy of FedAvg under different clipping thresholds.}\label{fig:impact}
% \vspace*{-5pt}
% \end{figure}
% %---------------------%

The SVHN dataset contains 73,257 training examples and 26,032 testing examples, and each of them is a  $32 \times 32$ colored image of digits from 0 to 9. We extend the training set with another 531,131 additional examples. For all algorithms except Fed-SMP-$\topk_k$, we evenly partition the training set across 6,000 clients. For Fed-SMP-$\topk_k$, we randomly sample 2,000 examples from the training set as the public dataset $D_p$ and evenly partition the remaining across 6,000 clients. We train a CNN model used in \cite{liang2020exploring} on SVHN dataset, which stacks two $5\times5$ convolution layers (the first with 64 filters, the second with 128 filters, each followed with $2\times2$ max pooling and ReLu activation), two fully connected layers (the first with 384 units, the second with 192 units, each followed with ReLu activation), and a final softmax output later, resulting in about $3.4$ million parameters in total. 
% \hl{This word only shows evenly or not, but not i.i.d...; The description should include both even vs non-even and i.i.d. vs non-iid. }

The Shakespeare dataset is a natural federated dataset built from \textit{The Complete Works of William Shakespeare}, where each of the total 715 clients corresponds to a speaking role with at least two lines. The training set and test set are obtained by splitting the lines from each client, and each client has at least one line for training or testing. We refer the reader to \cite{reddi2020adaptive} for more details on processing the raw data. Note that for Fed-SMP-$\topk_k$, we randomly select 1000 samples in total from the processed training set of clients as the public dataset $D_p$. For Shakespeare, we use the recurrent neural network (RNN) model in \cite{reddi2020adaptive} to predict the next work based on the preceding words in a line. The RNN model takes an 80-character sequence as input and consists of a $80\times8$ embedding layers and 2 LSTM layers (each with 256 nodes) followed by a dense layer, resulting in about 0.8 million parameters in total.

%---------------------%
\begin{table}[htpb]
% \vspace*{-5pt}
\renewcommand{\arraystretch}{1.2}
\caption{Summary of results on Shakespeare dataset.}\label{tab:shakespeare}
\begin{center}
\resizebox{\columnwidth}{!}{\begin{tabular}{c|l|c|c|c}
\hline
\multirow{1}{*}{Compression} & \multirow{2}{*}{Algorithm} & \multicolumn{3}{c}{Performance}  \\\cline{3-5} 
\multirow{1}{*}{ratio}  & & {Accuracy ($\%$)} & {Cost (MB)} & { Privacy ($\epsilon$)} \\\cline{1-5} 

% \multirow{2}{*}{$p = 0.001$} 
% & Fed-SMP-$\topk_k$ & $ $   & 0.21  & 1.0 \\\cline{2-5}
% % & FedAvg-$\topk_k$ & $81.48\pm 0.48$ & 0.02  & - \\\cline{2-5}
\multirow{2}{*}{$p = 0.005$} 
& Fed-SMP-$\topk_k$ & $55.41 \pm 0.20$  & 0.23  & 6.99 \\\cline{2-5} 
& Fed-SMP-$\randk_k$ & $36.41 \pm 0.15$  & 0.23  & 6.99\\\hline
\multirow{2}{*}{$p = 0.01$} 
& Fed-SMP-$\topk_k$ &$56.04 \pm 0.23$ & 0.46  & 6.99 \\\cline{2-5} 
& Fed-SMP-$\randk_k$ & $38.99 \pm 1.79$  & 0.46  & 6.99 \\\hline
\multirow{2}{*}{$p = 0.05$} 
& Fed-SMP-$\topk_k$ &  $\mathbf{56.76 \pm 0.13}$ & 2.30  & 6.99\\\cline{2-5} 
& Fed-SMP-$\randk_k$ & $50.79 \pm 0.27$  & 2.30  & 6.99 \\\hline 

\multirow{2}{*}{$p = 0.1$} 
& Fed-SMP-$\topk_k$ & ${55.74 \pm 0.06}$ & 4.60 & 6.99 \\\cline{2-5} 
& Fed-SMP-$\randk_k$ & $53.15 \pm 0.14$ & 4.60  & 6.99 \\\hline
\multirow{2}{*}{$p = 0.2$}  
& Fed-SMP-$\topk_k$ & $55.89 \pm 0.14$ & 9.20 & 6.99 \\\cline{2-5} 
& Fed-SMP-$\randk_k$ &  $\mathbf{54.22 \pm 0.15}$ & 9.20 & 6.99 \\\hline

\multirow{2}{*}{$p = 0.4$}  
& Fed-SMP-$\topk_k$ & $54.36 \pm 0.22$ & 18.41 & 6.99 \\\cline{2-5} 
& Fed-SMP-$\randk_k$ & ${53.68 \pm 0.14}$ & 18.41  & 6.99\\\hline
\multirow{2}{*}{$p = 0.8$}  
& Fed-SMP-$\topk_k$ & $52.14 \pm 0.31$ & 36.82 & 6.99 \\\cline{2-5} 
& Fed-SMP-$\randk_k$ & $51.85 \pm 0.29$ & 36.82  & 6.99 \\\hline
\multirow{2}{*}{$p = 1.0$} 
& DP-FedAvg &$51.03\pm0.28$ & 46.02 & 6.99 \\\cline{2-5}
& FedAvg & $62.77\pm0.05$ & 46.02 & - \\
\hline
\end{tabular}}
\end{center}
\vspace*{-10pt}
\end{table}
%---------------------%

\subsection{Experimental Settings}
For Fashion-MNIST and SVHN datasets, the server randomly samples $r = 100$ clients to participate in the training at each round, and for Shakespeare, $r=10$. For the local optimizer on clients, we use momentum SGD for both datasets. Specifically, for Fashion-MNIST, we set the local momentum coefficient as $0.5$, local learning rate as $\eta_l=0.125$ decaying at a rate of 0.99 at each communication round, batch size $B=10$ and local update period to be 10 epochs; for SVHN, we set the local momentum coefficient as $0.8$, local learning rate as $\eta_l=0.05$ decaying at a rate of 0.99 at each communication round, batch size $B=50$, local update period to be 5 epochs and number of communication rounds $T=180$; for Shakespeare, we set the local momentum coefficient as $0.9$, local learning rate as $\eta_l=1$ decaying at a rate of 0.99 at every 50 round, batch size $B=4$ and local update period to be 1 epoch and number of rounds $T=1000$.
% We observe that \hu{for Fashion-MNIST and SVHN }the baseline FedAvg converges and achieves a desirable accuracy after 180 rounds of training on both datasets, so we set the number of rounds for them as $T=180$ by default; \hu{and for Shakespeare, the baseline FedAvg converges after 1000 rounds, so we set $T=1000$ for it by default}. 
It is worth noting that the momentum at each client is initialized at the beginning of every communication round, since local momentum will be stale due to the partial participation of clients. Specially, for Fed-SMP-$\topk_k$, the server uses the same optimizer as the clients with the same local update period to compute the global model update $\Delta_p$.
% 
% and batch size of $B=10$ throughout the training process.  For all experiments, we set the local learning rate as $\eta_l=0.125$, decaying at a rate of 0.99 after each communication round, and number of local iterations $\tau = 10$. For Fed-SMP-$\topk_k$, the server uses the same optimizer as the clients to compute the global model update $\Delta_p$ with a batch size of $1000$ and 10 local iterations. %(better to enlarge the batch size equivalent to 100 clients times batch size per client and same number of local iterations??)}%, i.e., $10$ passes over the public dataset.

For the privacy-preserving algorithms, we compute the end-to-end privacy loss using the API provided in \cite{Opacus}. For all datasets, we follow \cite{mcmahan2017learning} to set $\delta=1/n^{1.1}$ such that $\delta$ is less than $1/n$. In particular, as it is challenging to obtain a small $\epsilon$ on Shakespeare dataset due to the small $n$ and the minimal $r$ sufficient for convergence, we follow \cite{andrew2021differentially} and \cite{mcmahan2017learning} to report privacy with a hypothetical population size $n=10^9$ for Shakespeare dataset. Moreover, the noise multiplier is set as $\sigma=1.4$ for Fashion-MNIST and SVHN datasets and $\sigma=0.3$ for Shakespeare dataset by default. We tune the clipping threshold over the grid $C=\{0.1, 0.2, 0.4, 0.6, 0.8, 1.0, 2.0\}$ and obtain the optimal clipping threshold, i.e., $C=1.0$ for Fashion-MNIST dataset, $C=0.4$ for SVHN and Shakespeare datasets.

%---------------------%

\begin{figure*}[htpb]
% \vspace*{-5pt}
\centering
\includegraphics[width=\textwidth]{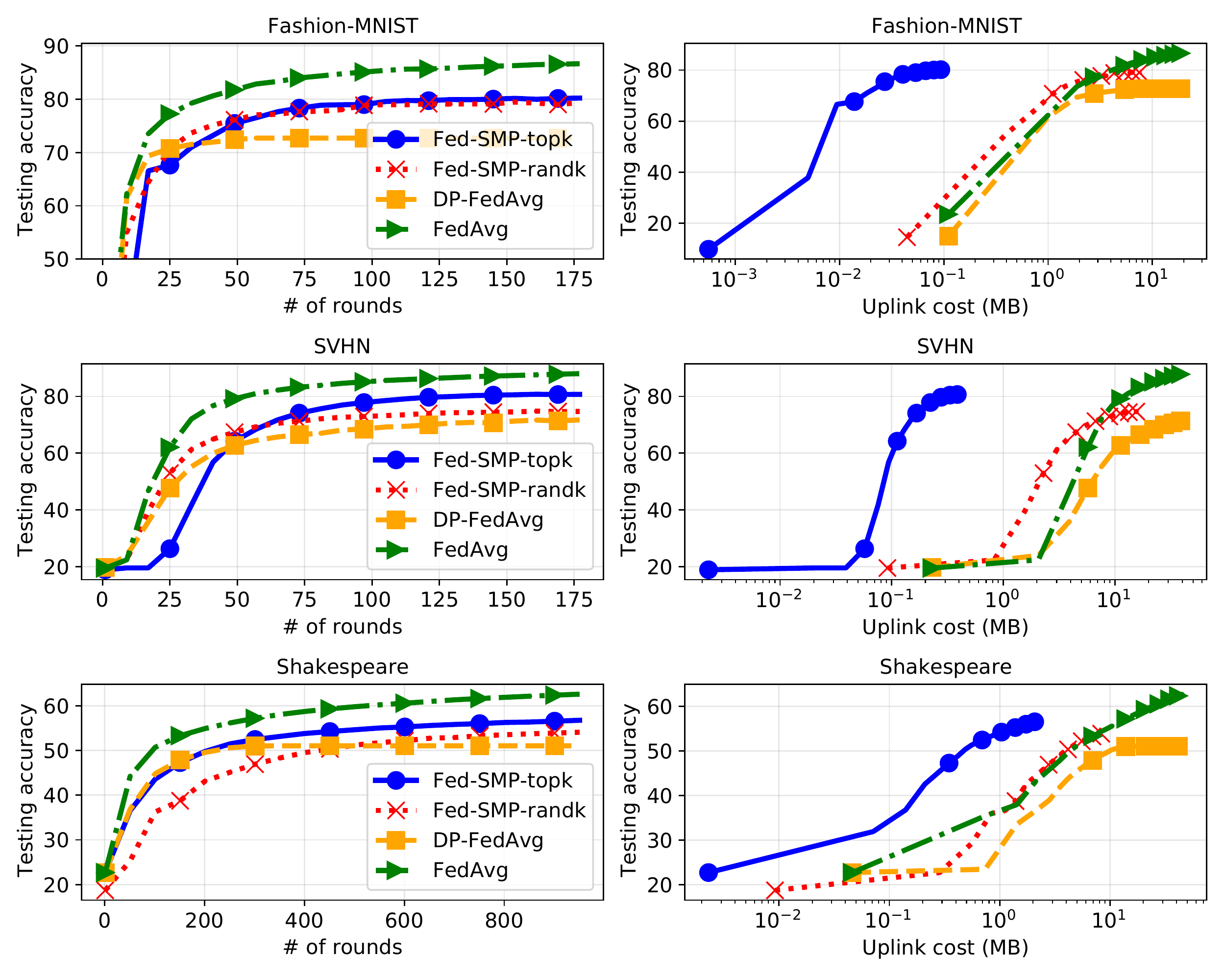}\label{fig:fmnist-communication}
% \subfloat[Fashion-MNIST dataset.]{\includegraphics[width=0.9\textwidth]{acc_round_cost_fmnist.pdf}\label{fig:fmnist-communication}}
% % \\
% \vspace*{-10pt}
% \subfloat[SVHN dataset.]{\includegraphics[width=0.9\textwidth]{acc_round_cost_svhn.pdf}\label{fig:svhn-communication}}
% \\
% \vspace*{-10pt}
% \subfloat[Shakespeare dataset.]{\includegraphics[width=0.9\textwidth]{acc_round_cost_svhn.pdf}\label{fig:shakespeare-communication}}
% \vspace*{-5pt}
\caption{Testing accuracy of Fed-SMP and DP-FedAvg w.r.t. communication round (left) and uplink communication cost (right).}\label{fig:com_}
\vspace*{-10pt}
\end{figure*}
In addition, we define the \textit{compression ratio} for $\randk_k$ and $\topk_k$ sparsifiers as $p := k / d$. When $p$ is smaller, more coordinates of the local model updates are set to zero, leading to a higher level of communication compression.

%%%%%%%%%%%%%%%%%%%%%%%%%%%%%%%%%%%%%%%%%
\subsection{Experiment Results}
%%%%%%%%%%%%%%%%%%%%%%%%%%%%%%%%%%%%%%%%%

Table~\ref{tab:fmnist}, Table~\ref{tab:svhn} and Table~\ref{tab:shakespeare} summarize the performance of FedAvg, DP-FedAvg and Fed-SMP after $T$ rounds on Fashion-MNIST, SVHN and Shakespeare, respectively. 
% 
% \hl{The column of \emph{Testing accuracy} reports the average ($\pm$ standard deviation) testing accuracy over 5 runs of the experiment with different random seeds. }
% 
We run each experiment with 5 random seeds and report the average and standard deviation of testing \emph{accuracy}. Note that we always report the best testing accuracy across all rounds in each experiment. We also report the uplink communication \emph{cost} that denotes the size of the data sent from a client to the server as the input into the secure aggregation protocol, which equals to $32 k \times T \times (r/n)$ bits for Fed-SMP and $32d \times T \times (r/n)$ bits for DP-FedAvg and FedAvg. The \emph{privacy} for DP-FedAvg and Fed-SMP is the accumulated privacy loss $\epsilon$ at the end of the training. 
% ---------------------%
\begin{figure*}[t]
\vspace*{-5pt}
\centering
\includegraphics[width=\textwidth]{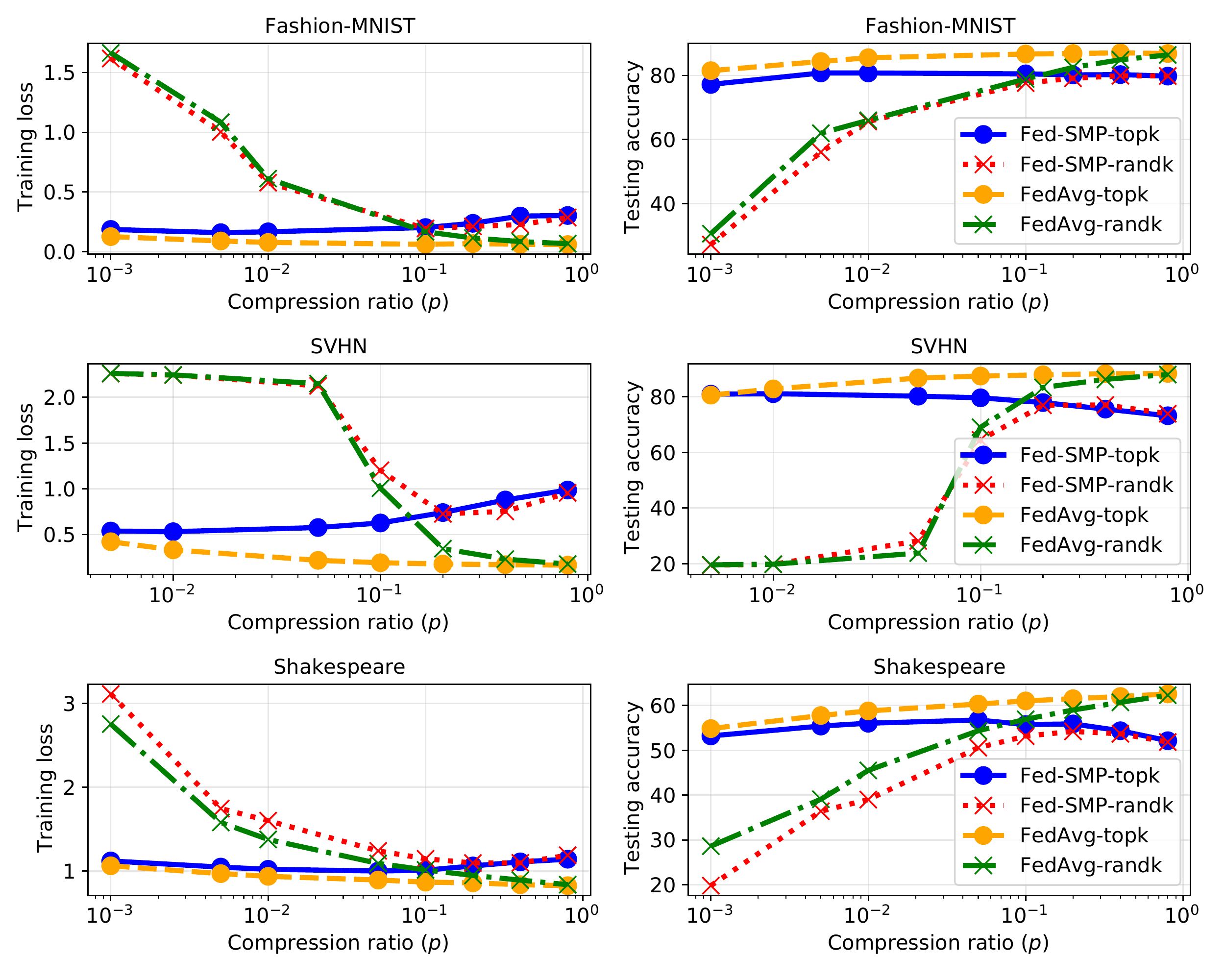}\label{fig:error_svhn}
% \subfloat[Fashion-MNIST dataset.]{\includegraphics[width=0.9\textwidth]{acc_loss_fmnist.pdf}\label{fig:error_fmnist}}
% \\
% \vspace*{-10pt}
% \subfloat[SVHN dataset.]{\includegraphics[width=0.9\textwidth]{acc_loss_svhn.pdf}\label{fig:error_svhn}}
% \\
% \vspace*{-10pt}
% \subfloat[Shakespeare dataset.]{\includegraphics[width=\textwidth]{acc_loss_p_shakespeare.pdf}\label{fig:error_shakespeare}}
% % \vspace*{-5pt}
\caption{Training loss and testing accuracy of Fed-SMP w.r.t compression ratio, compared with FedAvg-$\randk_k$/$\topk_k$.}\label{fig:error}
\vspace*{-10pt}
\end{figure*}
%---------------------%

% compare testing accuracy with private and non-private baseline (overall)
From Table~\ref{tab:fmnist}, we can see that the non-private FedAvg achieves a testing accuracy of $86.61\%$ on Fashion-MNIST dataset. When adding Gaussian noise to ensure $(1.01, n^{-1.1})$-DP, the accuracy drops to $71.40\%$ for DP-FedAvg. By choosing a proper value of $p$, Fed-SMP can reach a higher accuracy than DP-FedAvg under the same level of DP guarantee. For example, the highest testing accuracy of Fed-SMP-$\topk_k$ is $80.76\%$ when $p = 0.005$, outperforming DP-FedAvg by $11\%$; the highest testing accuracy of Fed-SMP-$\randk_k$ is $79.88\%$ when $p = 0.4$, outperforming DP-FedAvg by $10\%$. 

From Table~\ref{tab:svhn}, we can see that the accuracy of non-private FedAvg on SVHN dataset is $88.47\%$, and then it decreases to $72.46\%$ to achieve $(1.01, n^{-1.1})$-DP guarantee for DP-FedAvg. For Fed-SMP with the same level of DP guarantee, Fed-SMP-$\topk_k$ can achieve a testing accuracy of $81.12\%$ on SVHN dataset when $p = 0.01$, outperforming DP-FedAvg by $12\%$, and Fed-SMP-$\randk_k$ can achieve a testing accuracy of $77.07\%$ on SVHN dataset when $p = 0.4$, outperforming DP-FedAvg by $6\%$.

From the results of Shakespeare dataset in Table~\ref{tab:shakespeare}, we can see that the accuracy of non-private FedAvg is $62.77\%$ and then decreases to $51.03\%$ to achieve $(6.99, n^{-1.1})$-DP guarantee for DP-FedAvg. For Fed-SMP with the same level of DP guarantee, Fed-SMP-$\topk_k$ achieves a testing accuracy of $56.76\%$ on Shakespeare when $p = 0.05$, outperforming DP-FedAvg by $11\%$, and Fed-SMP-$\randk_k$ achieves a testing accuracy of $54.22\%$ on Shakespeare when $p = 0.2$, outperforming DP-FedAvg by $6\%$. In the following, we further evaluate Fed-SMP with more experimental settings.

% \hu{Besides, to demonstrate the necessity of using a public dataset for Fed-SMP-$\topk_k$, we also observe the performance of the global model trained on the public dataset only. Compared with Fed-SMP-$\topk_k$, the test accuracy of the model trained on the public dataset only is 
% }

\subsubsection{Communication efficiency of Fed-SMP}

% compare communication efficiency with baseline (overall) 
For the communication cost, we can observe from Table~\ref{tab:fmnist}-\ref{tab:shakespeare} that the uplink communication cost for Fed-SMP on both datasets is relatively lower than that of DP-FedAvg and FedAvg under the same privacy guarantee, since our scheme integrates well with the secure aggregation protocol and takes advantage of model update compression. To further demonstrate the communication efficiency of Fed-SMP, we show the convergence speed and communication cost of our Fed-SMP algorithms. Specifically, we select the Fed-SMP algorithms with the optimal compression ratio and show its testing accuracy with respect to the communication round and uplink communication cost in Fig.~\ref{fig:com_}, compared with DP-FedAvg and FedAvg: For Fashion-MNIST dataset, we draw the results of Fed-SMP-$\topk_k$ with $p=0.005$ and Fed-SMP-$\randk_k$ with $p=0.4$, while achieving $(1.01, n^{-1.1})$-DP; for SVHN dataset, we draw the results of Fed-SMP-$\topk_k$ with $p=0.01$ and Fed-SMP-$\randk_k$ with $p=0.4$, while achieving $(1.01, n^{-1.1})$-DP; for Shakespeare dataset, we draw the results of Fed-SMP-$\topk_k$ with $p=0.05$ and Fed-SMP-$\randk_k$ with $p=0.2$, while achieving $(6.99, n^{-1.1})$-DP. 

From Fig.~\ref{fig:com_}, we observe that DP-FedAvg converges to a lower accuracy than FedAvg due to the added DP noise. Fed-SMP converges slower than DP-FedAvg due to the use of sparsification. However, the final accuracy of Fed-SMP is higher than DP-FedAvg because of the advantage of sparsification in improving privacy and model accuracy. Moreover, Fig.~\ref{fig:com_} also demonstrates that Fed-SMP is more communication-efficient than DP-FedAvg, and Fed-SMP-$\topk_k$ is more communication-efficient than Fed-SMP-$\randk_k$. For instance, to achieve a target accuracy $72\%$ on Fashion-MNIST dataset, Fed-SMP-$\topk_k$ and Fed-SMP-$\randk_k$ save $99\%$ and $70\%$ of uplink communication cost compared with DP-FedAvg, respectively; to achieve a target accuracy $72\%$ on SVHN dataset, Fed-SMP-$\topk_k$ and Fed-SMP-$\randk_k$ save $99\%$ and $82\%$ of uplink communication cost compared with DP-FedAvg, respectively; to achieve a target accuracy $50\%$ on Shakespeare dataset, Fed-SMP-$\topk_k$ and Fed-SMP-$\randk_k$ save $95\%$ and $60\%$ of uplink communication cost compared with DP-FedAvg, respectively. %In addition, we also find that compared with FedAvg, Fed-SMP-$\topk_k$ saves the communication cost greatly. For example, %to achieve a target accuracy $80\%$ on Fashion-MNIST dataset, Fed-SMP-$\topk_k$ saves $98\%$ of uplink communication cost and achieves $(1.0, 10^{-4})$-DP at the same time, compared with FedAvg; to achieve a target accuracy $81\%$ on SVHN dataset, Fed-SMP-$\topk_k$ saves $97\%$ of uplink communication cost and achieves $(1.0, 10^{-4})$-DP compared with FedAvg; 

\subsubsection{Tradeoff between privacy and compression errors}

From Table~\ref{tab:fmnist}-Table~\ref{tab:shakespeare}, we have observed that as compression ratio $p$ decreases from 1.0 to a small value (e.g., 0.005), the testing accuracy of Fed-SMP first increases and then decreases. This is due to the change of privacy error and compression error in the convergence bound \eqref{eqn:convergen_result}, which is controlled by $k$. Since the compression ratio $p=k/d$, the smaller $p$ is, the lower the privacy error is, but the compression error could increase. In the following, we conduct additional experiments to further demonstrate this tradeoff under various compression ratios. As it is hard to directly show the tradeoff between privacy error and compression error, we study the difference between the performances of Fed-SMP and FedAvg-$\randk_k$/$\topk_k$, which differ only in DP noise addition.  %Note that as Fed-SMP achieves $(1, 10^{-4})$-DP at the 180-th round, we set the total communication rounds for FedAvg-$\randk_k$ and FedAvg-$\topk_k$ as 180 for fair comparisons.

Specifically, we show the training losses and testing accuracies of Fed-SMP, FedAvg-$\randk_k$ and FedAvg-$\topk_k$ with respect to different compression ratios in Fig.~\ref{fig:error}. Note that here Fed-SMP achieves $(1.01, n^{-1.1})$-DP for Fashion-MNIST and SVHN datasets and $(6.99, n^{-1.1})$-DP for Shakespeare dataset. From the results, we can see that as the compression ratio $p$ increases, the training losses of FedAvg-$\randk_k$ and FedAvg-$\topk_k$ on all datasets monotonically decrease, and the corresponding testing accuracies monotonically increase, due to the decreasing compression error incurred from sparsification. 
However, the training losses of Fed-SMP-$\randk_k$ and Fed-SMP-$\topk_k$ decrease at first and then increase, and the corresponding testing accuracies increase at first and then decrease. This matches our observations from Theorem~\ref{thm:convergence}, i.e., the performance of Fed-SMP depends on the tradeoff between the compression error and privacy error. When $p$ is too small (e.g., $p<0.1$ for Fed-SMP-$\randk_k$ on Fashion-MNIST datasets), despite the privacy amplification effect of compression that reduces the amount of added Gaussian noise, the compression error is significant and dominates the total convergence error, leading to a higher training loss and lower testing accuracy. 
% In this case, the model performance will be poor, e.g., the testing accuracy of Fed-SMP-$\randk_k$ on both datasets is low when $p<0.1$. 
As $p$ increases, the compression error decreases, leading to a lower training loss. However, if $p$ becomes too large, the privacy amplification effect of compression becomes negligible, and the privacy error starts to dominate the convergence error, leading to a higher training loss.
% In this case, the DP error (due to the large amount of added Gaussian noise) will dominate the convergence error and result in model performance's decreasing. 
For example, the testing accuracies of Fed-SMP-$\randk_k$ start to decrease when $p>0.4$ for Fashion-MNIST and SVHN datasets and $p>0.2$ for Shakespeare dataset, and the testing accuracy of Fed-SMP-$\topk_k$ starts to decrease when $p>0.005$ on Fashion-MNIST dataset, $ p > 0.01$ on SVHN dataset, and $p>0.05$ for Shakespeare dataset. 
Therefore, to achieve the best performance of Fed-SMP, the choice of $p$ needs to balance the privacy error and compression error. Furthermore, we find that Fed-SMP-$\topk_k$ tends to achieve its best performance when $p$ is small (i.e., 0.005, 0.01 and 0.1 for Fashion-MNIST, SVHN and Shakespeare datasets, respectively), and Fed-SMP-$\randk_k$ tends to achieve its best performance when $p$ is large (i.e., $0.4, 0.4, 0.2$ for Fashion-MNIST, SVHN and Shakespeare datasets, respectively). % Also, the best Fed-SMP-$\topk_k$ outperforms the best Fed-SMP-$\randk_k$ on both datasets. 

\subsubsection{Tradeoff between privacy and accuracy}

%---------------------%
\begin{table}[t]
% \vspace*{-5pt}
\renewcommand{\arraystretch}{1.2}
\caption{Privacy-accuracy tradeoff of Fed-SMP-$\randk_k$ with different compression ratios on Fashion-MNIST dataset.}\label{tab:tradeoff-randk-fmnist}
\begin{center}
\resizebox{\columnwidth}{!}{
\begin{tabular}{c|c|c|c|c|c}
\hline
\multirow{2}{*}{Privacy Loss} & 
\multicolumn{5}{c}{Testing Accuracy (\%)}  \\\cline{2-6}
&$p=0.005$ &$p=0.01$ & $p=0.1$ &$p=0.4$ &	$p=0.8$\\\cline{1-6}
$\epsilon=2.02$ & $	60.02$ & $	67.77 $ & $	77.76 $ & $	81.50$ & $\textbf{81.89}$ \\\cline{1-6}
$\epsilon=1.01$ & $ 56.04$ & 	$ 65.61$ &  $	77.59$ & $\textbf{79.88}$ & ${79.82}$ \\\cline{1-6}
$\epsilon=0.58$  & $	54.13	$ & $61.90	$ & $\textbf{75.80}	$ & ${69.62}$ & $53.39$ \\\cline{1-6}
$\epsilon=0.44$ & $	51.72 $ & $	59.08$ & $\textbf{74.54}$ & $	52.69$ & $	29.26$ \\\cline{1-6}
% $\epsilon=0.35$ & $	50.57$ & $	56.39$ & $\textbf{69.90}$ & $	34.64$ & $20.66$ \\\cline{1-6}
% $\epsilon=1.96$ & $	60.02 \pm 2.85 $ & $	64.95 \pm 2.82$ & $	77.76 \pm 0.50$ & $	81.50 \pm 0.48$ & $	81.89 \pm 0.32$ \\\cline{1-6}
% $\epsilon=1.00$ & $ 56.04 \pm 5.22$ & 	$ 65.61\pm 1.04$ &  $	77.59 \pm 0.53$ & $	79.88\pm 0.37$ & $ 79.82 \pm 0.35$ \\\cline{1-6}
% $\epsilon=0.57$  & $	54.13\pm 5.11	$ & $61.90\pm 2.04	$ & $75.80	\pm 0.59$ & $69.62	\pm 1.73$ & $53.39\pm 2.55$ \\\cline{1-6}
% $\epsilon=0.43$ & $	51.72 \pm 4.87$ & $	59.08\pm 3.27$ & $	74.54\pm 0.74$ & $	52.69\pm 3.04$ & $	29.26 \pm 3.44$ \\\cline{1-6}
% $\epsilon=0.35$ & $	50.57\pm 4.95$ & $	56.39\pm 3.74$ & $	69.90\pm 1.98$ & $	34.64\pm 8.23$ & $	20.66\pm3.44$ \\\cline{1-6}
\hline
\end{tabular}}
\end{center}
\vspace*{-10pt}
\end{table}
%---------------------%
%---------------------%
\begin{table}[htpb]
\vspace*{-5pt}
\renewcommand{\arraystretch}{1.2}
\caption{Privacy-accuracy tradeoff of Fed-SMP-$\topk_k$ with different compression ratios on Fashion-MNIST dataset.}\label{tab:tradeoff-topk-fmnist}
\begin{center}
\resizebox{\columnwidth}{!}{
\begin{tabular}{c|c|c|c|c|c}
\hline
\multirow{2}{*}{Privacy Loss} & 
\multicolumn{5}{c}{Testing Accuracy (\%)}  \\\cline{2-6}
&$p=0.001$ &$p=0.005$ & $p=0.01$ &$p=0.1$ &	$p=0.8$\\\cline{1-6}
$\epsilon=2.02$ & 78.11 &	81.10 &	81.62&	\textbf{82.04} &	81.71 \\\cline{1-6}
$\epsilon=1.01$ & 77.18&	80.76 &	\textbf{80.76} &	80.49&	80.16 \\\cline{1-6}
$\epsilon=0.58$  & 76.42&	\textbf{78.76} &	78.44&	75.04&	56.02 \\\cline{1-6}
$\epsilon=0.44$ & 74.55&\textbf{76.90} &	75.77&	67.79&	28.55 \\\cline{1-6}
\hline
\end{tabular}}
\end{center}
\vspace*{-10pt}
\end{table}
%---------------------%
%---------------------%
\begin{table}[htpb]
\vspace*{-10pt}
\renewcommand{\arraystretch}{1.2}
\caption{Privacy-accuracy tradeoff of Fed-SMP-$\randk_k$ with different compression ratios on SVHN dataset.}\label{tab:tradeoff-randk-svhn}
\begin{center}
\resizebox{\columnwidth}{!}{
\begin{tabular}{c|c|c|c|c|c}
\hline
\multirow{2}{*}{Privacy Loss} & 
\multicolumn{5}{c}{Testing Accuracy (\%)}  \\\cline{2-6}
&$p=0.05$ &$p=0.1$ & $p=0.2$ &$p=0.4$ &	$p=0.8$\\\cline{1-6} 
$\epsilon=2.02$ & 28.18&65.85&	79.76&	\textbf{81.35}&	80.39 \\\cline{1-6}
$\epsilon=1.01$ & 28.15 &	64.50 &	76.94 & \textbf{77.07} & 73.90 \\\cline{1-6}
$\epsilon=0.58$ & 26.30 &60.21 & \textbf{69.90} &65.75 &	61.39 \\\cline{1-6}
$\epsilon=0.44$ & 26.07 &55.37 & \textbf{61.69} & 54.81	& 24.99 \\\cline{1-6}
\hline
\end{tabular}}
\end{center}
\vspace*{-10pt}
\end{table}
%---------------------%
%---------------------%
\begin{table}[htpb]
\vspace*{-5pt}
\renewcommand{\arraystretch}{1.2}
\caption{Privacy-accuracy tradeoff of Fed-SMP-$\topk_k$ with different compression ratios on SVHN dataset.}\label{tab:tradeoff-topk-svhn}
\begin{center}
\resizebox{\columnwidth}{!}{
\begin{tabular}{c|c|c|c|c|c}
\hline
\multirow{2}{*}{Privacy Loss} & 
\multicolumn{5}{c}{Testing Accuracy (\%)}  \\\cline{2-6}
&$p=0.05$ &$p=0.1$ & $p=0.2$ &$p=0.4$ &	$p=0.8$\\\cline{1-6} 
$\epsilon=2.02$ & 81.24 &	81.86 &	\textbf{83.44} &	81.50 &	80.08 \\\cline{1-6}
$\epsilon=1.01$ & 80.94 &	\textbf{81.12} &	80.22 &	75.58 &	73.18 \\\cline{1-6}
$\epsilon=0.58$ & 78.42 &	\textbf{78.85} &	75.07 &	62.53 &	59.50 \\\cline{1-6}
$\epsilon=0.44$ & \textbf{78.76} &	76.45 &	67.11 &	32.32 &	22.92 \\\cline{1-6}
\hline
\end{tabular}}
\end{center}
\vspace*{-10pt}
\end{table}
%---------------------%

%---------------------%
\begin{table}[htpb]
\vspace*{-10pt}
\renewcommand{\arraystretch}{1.2}
\caption{Privacy-accuracy tradeoff of Fed-SMP-$\randk_k$ with different compression ratios on Shakespeare dataset.}\label{tab:tradeoff-randk-shakespeare}
\begin{center}
\resizebox{\columnwidth}{!}{
\begin{tabular}{c|c|c|c|c|c}
\hline
\multirow{2}{*}{Privacy Loss} & 
\multicolumn{5}{c}{Testing Accuracy (\%)}  \\\cline{2-6}
&$p=0.05$ &$p=0.1$ & $p=0.2$ &$p=0.4$ &	$p=0.8$\\\cline{1-6} 
$\epsilon=11.51$ & 50.87 & 53.16 &	54.78 &	\textbf{55.34} &	54.30 \\\cline{1-6}
$\epsilon=6.99$ & 50.79 &	53.15 &	\textbf{54.22} & 53.68 & 51.85 \\\cline{1-6}
$\epsilon=4.70$ & 50.59 &52.87& \textbf{53.26} & 51.73 &	49.36 \\\cline{1-6}
$\epsilon=3.39$ & 50.35 & \textbf{52.39} & 51.93 & 49.74	& 46.18 \\\cline{1-6}
\hline
\end{tabular}}
\end{center}
\vspace*{-10pt}
\end{table}
%---------------------%

%---------------------%
\begin{table}[htpb]
\vspace*{-5pt}
\renewcommand{\arraystretch}{1.2}
\caption{Privacy-accuracy tradeoff of Fed-SMP-$\topk_k$ with different compression ratios on Shakespeare dataset.}\label{tab:tradeoff-topk-shakespeare}
\begin{center}
\resizebox{\columnwidth}{!}{
\begin{tabular}{c|c|c|c|c|c}
\hline
\multirow{2}{*}{Privacy Loss} & 
\multicolumn{5}{c}{Testing Accuracy (\%)}  \\\cline{2-6}
&$p=0.005$ &$p=0.01$ &$p=0.05$ & $p=0.1$ &$p=0.2$ \\\cline{1-6} 
$\epsilon=11.51$ & 55.43 & 56.10 & 57.05 &	\textbf{57.30} &	56.97 \\\cline{1-6}
$\epsilon=6.99$ & 55.41 & 56.04	& \textbf{56.76} &	55.74 &	55.89  \\\cline{1-6}
$\epsilon=4.70$ & 55.34 & 56.02	& \textbf{56.24}	& 54.85	& 54.61 \\\cline{1-6}
$\epsilon=3.39$ & 55.00 & \textbf{55.87} & 54.97	& 54.34	& 53.00  \\\cline{1-6}
\hline
\end{tabular}}
\end{center}
\vspace*{-10pt}
\end{table}
%---------------------%

In this part, we study the tradeoff between privacy and model accuracy of Fed-SMP. In Table~\ref{tab:tradeoff-randk-fmnist}-\ref{tab:tradeoff-topk-svhn}, we show the testing accuracy of Fed-SMP with respect to the privacy loss, under different compression ratios. Note that here we vary the noise multiplier $\sigma$ to achieve different privacy losses. We can see that under the same compression ratio $p$, the testing accuracy of Fed-SMP always decreases when the privacy loss $\epsilon$ decreases. For example, the testing accuracy of Fed-SMP-$\randk_k$ with $p=0.005$ decreases from $60.02\%$ to $51.72\%$ on Fashion-MNIST dataset when the privacy loss $\epsilon$ decreases from 2.02 to 0.44. The reason is that as $\epsilon$ decreases, the noise multiplier $\sigma$ increases (according to Theorem~\ref{thm:privacy_loss}), and hence, the convergence error increases (according to Theorem~\ref{thm:convergence}). Furthermore, we also observe that under the same privacy loss, the testing accuracy of Fed-SMP always increases at first and then decreases, as the compression ratio $p$ increases. For example, when the privacy loss $\epsilon = 2.02$, the testing accuracy of Fed-SMP-$\randk_k$ on SVHN dataset increases from $28.18\%$ to $81.35\%$ and then decreases to $80.39\%$, as the compression ratio $p$ increases from 0.05 to 0.8. This matches with our observations from Fig.~\ref{fig:error}, i.e., there exists an optimal compression ratio that achieves the highest testing accuracy under the same privacy loss (as highlighted in the table) due to the tradeoff between privacy and compression errors.

Besides, the optimal compression ratio decreases as the privacy loss decreases. For example, from Table~\ref{tab:tradeoff-topk-shakespeare}, the optimal $p$ for Fed-SMP-$\topk_k$ on Shakespeare dataset is 0.1 when $\epsilon=11.51$, and as $\epsilon$ decreases to 3.39, the optimal $p$ decreases to 0.01. We observe similar trends on Table~\ref{tab:tradeoff-randk-fmnist}-\ref{tab:tradeoff-randk-shakespeare}. It is because that the optimal $p$ always balances the tradeoff between privacy and compression errors to minimize the overall convergence error. As $\epsilon$ decreases, the noise multiplier $\sigma$ increases, and thus, the privacy error increases. In this case, the optimal $p$ should decrease to reduce the privacy error.

\section{Related Work}\label{sec:related}
%%%%%%%%%%%%%%%%%%%%%%%%%%%%%%%%%%%%%%%%%%%%%%%%%%%%%%%%%%%%

% As machine learning are playing more and more important roles in all sectors, privacy issues of the data source are receiving increasing attention. Prior works on protecting the privacy of training data rely on costly mechanisms such as secure multi-party computation or homomorphic encryption. Due to their high computation and communication overhead, such mechanisms can only be applied to simple learning tasks such as linear regression\cite{NiWe13} and logistic regression\cite{MoZh17}. Moreover, they only protect privacy during the training process and do not prevent privacy leakage in the final learned model, which has been shown to be feasible by a number of attacks, such as model inversion attacks and membership inference attacks\cite{fredrikson2015model,al2016reconstruction,shokri2017membership}. 

%DP provides formal privacy guarantee without making assumptions on the adversary and thus has been widely adopted in differentially private learning. Most DP mechanisms provide DP guarantee by injecting carefully calibrated noises (e.g, Gaussian, Laplacian or Binomial noise) during the learning process \cite{abadi2016deep}. 

FL, or distributed learning in general, with DP has attracted increasing attention recently. Different from centralized learning, FL involves multiple communication rounds between clients and the server; therefore, DP needs to be preserved for all communication rounds. The main challenge of achieving DP in FL lies in maintaining high model accuracy under a reasonable DP guarantee. %
%
%Larger amount of noises need to be injected for better DP guarantee for the same algorithm. 
%
Prior related works rely on specialized techniques such as shuffling \cite{liu2020flame,erlingsson2019amplification,girgis2021shuffled}, sparsification\cite{liu2020fedsel,lyu2021dp,hu2021federated}, and aggregation directly over wireless channels \cite{seif2020wireless} to boost DP guarantee with the same amount of noise, which are only applicable to some specific settings. For instance, the model shuffling method in \cite{liu2020flame} uses a shuffler to permute the local model updates before aggregation, improving the DP guarantee for protecting the local model update of each client. In \cite{liu2020fedsel}, the top$_k$ sparsification is integrated with record-level DP to protect the local model update of a client, and in \cite{hu2021federated} random sparsification is used with gradient perturbation to achieve record-level DP. Our approach is orthogonal to theirs as we aim to achieve client-level DP that protects the local dataset of a client and can be combined to achieve even better privacy-accuracy tradeoffs. Moreover, our work theoretically analyzes the convergence of the proposed methods to demonstrate the benefit of integrating biased/unbaised sparsification with DP in the classic FedAvg algorithm in a rigorous manner, filling the gap in the state-of-arts. Some works \cite{guo2018practical,huang2019dp} also develop differentially private versions of distributed algorithms that require fewer iterations than SGD-based algorithms to converge such as alternating direction method of multipliers, but they are only applicable to convex settings and do not allow partial client participation at each communication round, which is a key feature in FL. 

Client-level DP provides more realistic protections than record-level DP against information leakage in the FL setting. With client-level DP, the participation of a client rather than a single data record needs to be protected, therefore requiring the addition of a larger amount of noise than record-level DP to achieve the same level of DP \cite{mcmahan2017learning,wei2021user}. Our work follows this line of research and proposes to integrate SMP with FedAvg to achieve higher model accuracy than prior schemes, under the same client-level DP guarantee in FL. Note there are other approaches to preserve privacy in distributed learning, such as secure multi-party computation \cite{kanagavelu2020two} or homomorphic encryption \cite{aono2017privacy}. However, those approaches have different privacy goals and cannot protect against information leakage incurred from observing the final trained models at the server. 

%However, it is still a long way to go before DP-based algorithms can be practically used without greatly impacting model accuracy \cite{wei2021user,truex2019hybrid}. 

%client-level dp vs record level dp

%Researchers are just starting to investigate the problem of achieving client-level DP in FL \cite{wei2021user,zhu2020voting}. 

% As machine learning are playing more and more important roles in all sectors, privacy issues of the data source are receiving increasing attention. Prior works on protecting the privacy of training data rely on costly mechanisms such as secure multi-party computation or homomorphic encryption. Due to their high computation and communication overhead, such mechanisms can only be applied to simple learning tasks such as linear regression\cite{NiWe13} and logistic regression\cite{MoZh17}. Moreover, they only protect privacy during the training process and do not prevent privacy leakage in the final learned model, which has been shown to be feasible by a number of attacks, such as model inversion attacks and membership inference attacks\cite{fredrikson2015model,al2016reconstruction,shokri2017membership}. 

Besides privacy, another bottleneck in FL is the high communication cost of transmitting model parameters. Typical techniques to address this issue include quantization and sparsification \cite{alistarh2017qsgd,bernstein2018signsgd, wang2018atomo, beguier2020efficient,chen2020scalecom}. For example, in \cite{beguier2020efficient}, sparsification and quantization are integrated with secure aggregation to mitigate the communication cost and privacy concern of FL simultaneously; in \cite{chen2020scalecom} the top-k coordinates of a client's gradient are selected cyclically for all clients to do gradient sparsification to save the communication cost in distributed learning. However, these works do not consider the rigorous privacy protection for clients. Some recent works \cite{agarwal2018cpsgd,zhang2020private,kerkouche2021compression} start to jointly consider communication efficiency and DP in distributed learning. Agarwal et al. \cite{agarwal2018cpsgd} propose cpSGD by making modifications to distributed SGD to make the method both private and communication-efficient. However, the authors treat these as separate issues and develop different approaches to address each within cpSGD. As shown in \cite{li2019privacy}, by separately reducing communication and enforcing privacy, errors in cpSGD are compounded and higher than considering privacy only. In comparison, our proposed Fed-SMP scheme considers those two issues jointly and provides substantially higher accuracy by leveraging sparsification to amplify privacy protection. Zhang et al.\cite{zhang2020private} also combine sparsification with Gaussian mechanism to achieve communication efficiency and DP in decentralized SGD, but their scheme adds noise before using sparsification and cannot leverage sparsification to improve privacy protection. Kerkouche et al.\cite{kerkouche2021compression} propose to use compressive sensing before adding noise to model updates in FL but relies on the strong assumption that model updates only have a few non-zero coordinates, which seldom holds in practice. 

%Finally, recent work called cpSGD [3] proposes modifications to distributed SGD to make the method both private and communication-efficient. However, the authors treat these as separate issues, and develop different approaches to address each within cpSGD. In §5 we compare directly with cpSGD and demonstrate that by separately reducing communication and enforcing privacy, errors in cpSGD are compounded; by instead considering these jointly, DiffSketch provides substantially higher accuracy.
%Furthermore, our proposed scheme has a rigorous performance guarantee. 
%

%%%%%%%%%%%%%%%%%%%%%%%%%%%%%%%%%%%%%%%%%%%%%%%%%%%%%%%%%%%%
\section{Conclusions}\label{sec:con}
%%%%%%%%%%%%%%%%%%%%%%%%%%%%%%%%%%%%%%%%%%%%%%%%%%%%%%%%%%%%

This paper has proposed Fed-SMP, a new FL scheme based on sparsified model perturbation. Fed-SMP achieves client-level DP while maintaining high model accuracy and is also communication-efficient. We have rigorously analyzed the convergence and end-to-end DP guarantee of the proposed scheme and extensively evaluated its performance on two common benchmark datasets. Experimental results have shown that Fed-SMP with both $\randk_k$ and $\topk_k$ sparsification strategies can improve the privacy-accuracy tradeoff and communication efficiency simultaneously compared with the existing methods. In the future, we will consider other compression methods such as low-rank approximation and quantization, and extend the scheme to alternative FL settings such as shuffled model.  
% In the future, we will consider optimizing the tradeoffs among privacy, accuracy, and communication in Fed-SMP. 

% \hl{(Need to check the reference below one by one to have the right format such as captilization. There are obvious errors there including the abbreviation in title and the publication year. Also, for Arxiv papers, should have their formal publication places if published, not only old Arxiv links. Google Scholar is not accurate. Reference is also too long for a journal paper. May just delete some trivial one.)}

\ifCLASSOPTIONcaptionsoff
  \newpage
\fi
\bibliographystyle{IEEEtran}
%\bibliography{./bibs/guo.bib,./bibs/gong.bib,./bibs/reference.bib}
\bibliography{hu,guo_career,gong}

\onecolumn
\appendices

\section{Proof of Theorem~\ref{thm:privacy_loss}}\label{subsec:proof_priv}

Suppose the client is sampled without replacement with probability $q=r/n$ at each round. By Lemma~\ref{lemma:gaussian_mechanism} and Lemma~\ref{lemma:rdp_sub}, the $t$-th round of Fed-SMP satisfies $(\alpha, \rho_t(\alpha))$-RDP, where 
\begin{equation}
\label{eqn:perround}
    \rho_t(\alpha) = \frac{3.5q^2\alpha}{\sigma^2},
\end{equation}
if $\sigma^2 \geq 0.7$ and $\alpha \leq 1+ (2/3)C^2\sigma^2\log(1/q\alpha(1+ \sigma^2)) $. Then by Lemma~\ref{lemma:rdp_comr}, Fed-SMP satisfies $(\alpha, T\rho_t(\alpha))$-RDP after $T$ rounds of training. Next, in order to guarantee $(\epsilon,\delta)$-DP according to Lemma~\ref{lemma:rdp_dp}, we need 
\begin{equation}
\label{eqn:simplify}
    \frac{3.5q^2 T\alpha}{\sigma^2} + \frac{\log(1/\delta)}{\alpha-1} \leq \epsilon.
\end{equation}
Suppose $\alpha$ and $\sigma$ are chosen such that the conditions for \eqref{eqn:perround} are satisfied. Choose $\alpha = 1 + 2\log(1/\delta)/\epsilon$ and rearrange the inequality in \eqref{eqn:simplify}, we need
\begin{equation}
    \sigma^2 \geq \frac{7q^2T(\epsilon + 2\log(1/\delta))}{\epsilon^2}.
\end{equation}
Then using the constraint on $\epsilon$ concludes the proof.

\section{Proof of Theorem~\ref{thm:convergence} with $\randk_k$ sparsifier}\label{subsec:proof_conve_randk}
\subsection{Notations}
For ease of expression, let $\bm{\theta}_i^{t,s}$ denote client $i$'s local model at local iteration $s$ of round $t$, and let $\bm{b}_i^t$ denote the Gaussian noise added to the sparsified model update where $[\bm{b}_i^t]_j\sim \mathcal{N}(0,C^2\sigma^2/r)$ if $[\bm{m}^t]_j=0$ and $[\bm{b}_i^t]_j=0$ otherwise. Let $\spar$ denote the sparsifier. In Fed-SMP, the update rule of the global model can be summarized as:
\begin{equation}
    \bm{\theta}^{t+1} = \bm{\theta}^{t} -  \frac{1}{r}\sum_{i\in\mathcal{W}^t} \Delta_i^t, \text{ where }  \Delta_i^t = \spar(\eta_l\sum_{s=0}^{\tau-1} \bm{g}_i^{t,s}) + \bm{b}_i^t 
\end{equation}
where $\Delta_i^t$ represents the sparsified noisy model update of client $i$. Assume the clients are sampled uniformly at random without replacement, then we can directly validate that the client sampling is unbiased:
\begin{equation}
    \mathbb{E}_{\mathcal{W}^t}\left[ \frac{1}{r}\sum_{i\in\mathcal{W}^t} \Delta_i^t\right] =\frac{1}{r}\sum_{\substack{\mathcal{W}\in[n],\\|\mathcal{W}|=r}} \mathbb{P}(\mathcal{W}^t=\mathcal{W}) \sum_{i\in\mathcal{W}^t} \Delta_i^t  =  \frac{1}{n}\sum_{i=1}^{n} \Delta_i^t.
\end{equation}
Moreover, let $\nabla f_i(\bm{\theta}_i^{t,s})$ represent the local gradient so that $\mathbb{E}_{\xi_i^{t,s}}[\bm{g}_i^{t,s}]=\nabla f_i(\bm{\theta}_i^{t,s})$. 
% Given the unbiasedness of client sampling and stochastic gradient, we have
% \begin{equation}
%     \mathbb{E}_t\left[\frac{1}{r} \sum_{i\in\mathcal{W}^t}\spar(\sum_{s=0}^{\tau-1} g_i^s(\bm{\theta}_i^{s})) \right] = \frac{1}{n} \sum_{i=1}^{n}\sum_{s=0}^{\tau-1} \nabla f_i(\bm{\theta}_i^{s})
% \end{equation}
For ease of expression, we let $\bm{d}_i^t = (1/\tau)\sum_{s=0}^{\tau-1} \bm{g}_i^{t,s} $ and $ \bm{h}_i^t =  (1/\tau)\sum_{s=0}^{\tau-1}  \nabla f_i(\bm{\theta}_i^{t,s})$, so we have $\Delta_i^t = \eta_l\tau\spar(\bm{d}_i^t) + \bm{b}_i^t $ and $ \mathbb{E}_t[\bm{d}_i^t] = \bm{h}_i^t$.

\subsection{Useful Inequalities}

\begin{lemma}[Jensen's inequality]\label{lemma:jensen}
For arbitrary set of $n$ vectors $\{\bm{a}_i\}_{i=1}^{n}, \bm{a}_i\in\mathbb{R}^d$ and positive weights $\{w_i\}_{i\in[n]}$,
\begin{equation}
    \left\|\sum_{i=1}^{n}w_{i}\bm{a}_i\right\|^2\leq \frac{\sum_{i=1}^{n}w_i\|\bm{a}_i\|^2}{\sum_{i=1}^{n}w_i}.
\end{equation}
\end{lemma}

\begin{lemma}\label{lemma:aaaa}
For arbitrary set of $n$ vectors $\{\bm{a}_i\}_{i=1}^{n}, \bm{a}_i\in\mathbb{R}^d$,
\begin{equation}
    \left\|\sum_{i=1}^{n}\bm{a}_i\right\|^2\leq n{\sum_{i=1}^{n}\|\bm{a}_i\|^2}.
\end{equation}
\end{lemma}
%---------------------------------%
\begin{lemma}\label{lemma:a+b}
For given two vectors $\bm{a}, \bm{b}\in\mathbb{R}^d$,
\begin{equation}
    \|\bm{a}+\bm{b}\|^2\leq (1+\alpha)\|\bm{a}\|^2 + (1+\alpha^{-1})\|\bm{b}\|^2, \forall \alpha>0.
\end{equation}
% This inequality also holds for the sum of two matrices, $\mathbf{A},\mathbf{B}\in\mathbb{R}^{n\times d}$ in Frobenius norm.
\end{lemma}

\begin{lemma}\label{lemma:ab}
For given two vectors $\bm{a}, \bm{b}\in\mathbb{R}^d$,
\begin{equation}
  2\langle \bm{a}, \bm{b}\rangle \leq \gamma\|\bm{a}\|^2 + \gamma^{-1}\|\bm{b}\|^2, \forall \gamma>0.
\end{equation}
\end{lemma}

\subsection{Detailed Proof}
Let $\spar:=(d/k)\randk_k$ denote the unbiased random sparsifier. By the $L$-smoothness of function $f$, we have
\begin{align}
    \nonumber
    \mathbb{E}_t[f(\bm{\theta}^{t+1}) -f(\bm{\theta}^t)] & \leq  \mathbb{E}_t\left\langle \nabla f(\bm{\theta}^{t}),   \bm{\theta}^{t+1} - \bm{\theta}^t\right\rangle + \frac{L}{2}\mathbb{E}_t\left\| \bm{\theta}^{t+1} - \bm{\theta}^t\right\|^2
    \\\nonumber
    & = -\mathbb{E}_t\left\langle \nabla f(\bm{\theta}^{t}),   \mathbb{E}_{\mathcal{W}^t}\left[\frac{1}{r} \sum_{i\in\mathcal{W}^t} \Delta_i^t\right] \right\rangle + \frac{L}{2}\mathbb{E}_t\left\|  \frac{1}{r} \sum_{i\in\mathcal{W}^t} \Delta_i^t\right\|^2\\\nonumber
    & = -\mathbb{E}_t\left\langle \nabla f(\bm{\theta}^{t}),\frac{1}{n} \sum_{i=1}^{n}\Delta_i^t \right\rangle + \frac{L}{2}\mathbb{E}_t\left\|  \frac{1}{r} \sum_{i\in\mathcal{W}^t} \Delta_i^t\right\|^2\\\label{eqn:A}
    % & = -\left\langle \nabla f(\bm{\theta}_{t}),\mathbb{E}_t\left[\frac{1}{n} \sum_{i=1}^{n}\left(\eta_l\tau\spar(\bm{d}_i^t) + \bm{b}_i^t\right) \right]\right\rangle + \frac{L}{2}\mathbb{E}_t\left\|  \frac{1}{r} \sum_{i\in\mathcal{W}^t}\left( \eta_l\tau\spar(\bm{d}_i^t) + \bm{b}_i^t\right) \right\|^2\\
    & = \underbrace{-\mathbb{E}_t\left\langle \nabla f(\bm{\theta}^{t}),\frac{1}{n} \sum_{i=1}^{n}\left(\eta_l\tau\spar(\bm{d}_i^t) + \bm{b}_i^t\right) \right\rangle}_{T_1} + \underbrace{ \frac{L}{2}\mathbb{E}_t\left\|  \frac{1}{r} \sum_{i\in\mathcal{W}^t}\left( \eta_l\tau\spar(\bm{d}_i^t) + \bm{b}_i^t\right) \right\|^2}_{T_2}
\end{align}
where the expectation $\mathbb{E}_t[\cdot]$ is taken over the sampled clients $\mathcal{W}^t$ and mini-batches $\xi_i^s,\forall i\in[n], s\in\{0,\dots,\tau-1\}$ at round $t$ and the sparsifier $\spar$. 
% \hl{(How does the client sampling effect show up in the proof? It seems that the proof does not use any client sampling? The second equality in the above equation seems to be over-simplifed and may contain error. Need to verify with other reference papers to see how they consier sampling. It is not trivial.)}.
As $\spar=(d/k)\randk_k$, due to the unbiasedness of the stochastic gradient and Gaussian noise, we have
\begin{align}
\nonumber
    T_1 & = -\left\langle \nabla f(\bm{\theta}^{t}),\mathbb{E}_t\left[\frac{1}{n} \sum_{i=1}^{n}\eta_l\tau\bm{d}_i^t  \right]\right\rangle -\left\langle \nabla f(\bm{\theta}^{t}),   \mathbb{E}_t\left[\frac{1}{n} \sum_{i=1}^{n}\bm{b}_i^t \right]\right\rangle \\\nonumber
    & = -{\eta_l\tau}\mathbb{E}_t\left\langle \nabla f(\bm{\theta}^{t}),   \frac{1}{n} \sum_{i=1}^{n}\bm{h}_i^t  \right\rangle\\\nonumber
    & = -\frac{\eta_l\tau }{2}\left\|  \nabla f(\bm{\theta}^{t})\right\|^2 - \frac{\eta_l\tau }{2}\mathbb{E}_t\left\| \frac{1}{n} \sum_{i=1}^{n}\bm{h}_i^t  \right\|^2 
    + \frac{\eta_l\tau }{2} \mathbb{E}_t\left\| \nabla f(\bm{\theta}^{t}) -  \frac{1}{n} \sum_{i=1}^{n}\bm{h}_i^t \right\|^2 \\\nonumber
    &\leq -\frac{\eta_l\tau }{2}\left\|  \nabla f(\bm{\theta}^{t})\right\|^2 + \frac{\eta_l\tau }{2} \mathbb{E}_t\left\| \frac{1}{n} \sum_{i=1}^{n} \frac{1}{\tau}\sum_{s=0}^{\tau-1}\left( \nabla f_i(\bm{\theta}^{t}) - \nabla f_i(\bm{\theta}_i^{t,s}) \right) \right\|^2 \\\nonumber
    & \leq -\frac{\eta_l\tau }{2}\left\|  \nabla f(\bm{\theta}^{t})\right\|^2 + \frac{\eta_l\tau }{2n} \sum_{i=1}^{n} \frac{1}{\tau}\sum_{s=0}^{\tau-1} \mathbb{E}_t\left\| \nabla f_i(\bm{\theta}^{t}) - \nabla f_i(\bm{\theta}_i^{t,s}) \right\|^2\\\label{eqn:T_1_randk}
    & \leq -\frac{\eta_l\tau }{2}\left\|  \nabla f(\bm{\theta}^{t})\right\|^2 + \frac{\eta_l  L^2}{2n} \sum_{i=1}^{n}\sum_{s=0}^{\tau-1} \mathbb{E}_t\left\| \bm{\theta}^{t}- \bm{\theta}_i^{t,s} \right\|^2,
    % \\\nonumber
    % & = -\frac{\eta_l\tau }{2}\left\|  \nabla f(\bm{\theta}^{t})\right\|^2 - \frac{\eta_l\tau}{2}\mathbb{E}_t\left\| \frac{1}{n} \sum_{i=1}^{n}\bm{h}_i^t  \right\|^2 
    % + \frac{\eta_l\tau}{2} \mathbb{E}_t\left\| \nabla f(\bm{\theta}^{t}) -  \frac{1}{n} \sum_{i=1}^{n}\bm{h}_i^t \right\|^2 \\\label{eqn:t1}
    % & \leq -\frac{\eta_l\tau}{2}\left\|  \nabla f(\bm{\theta}^{t})\right\|^2 + \frac{\eta_l\tau}{2} \mathbb{E}_t\left\| \nabla f(\bm{\theta}^{t}) -  \frac{1}{n} \sum_{i=1}^{n}\bm{h}_i^t \right\|^2,
\end{align}
where the first inequality results from the fact that $\|(1/n)\sum_{i=1}^{n}\bm{h}_i^t \|^2\geq 0$, the second inequality uses Lemma~\ref{lemma:aaaa}, and the third inequality uses the $L$-smoothness of function $f_i$.

For $T_2$, let $\phi_k:= d/k-1$, we have
\begin{align}
\nonumber
    T_2 & \leq \frac{L\eta_l^2\tau^2}{2}\mathbb{E}_t\left[\frac{1}{r} \sum_{i\in\mathcal{W}^t}\left\| \spar(\bm{d}_i^t) \right\|^2\right] + \frac{L}{2}\mathbb{E}_t\left\|  \frac{1}{r} \sum_{i\in\mathcal{W}^t} \bm{b}_i^t \right\|^2\\\nonumber
    & = \frac{L\eta_l^2\tau^2}{2n} \sum_{i=1}^{n}\mathbb{E}_t\left\|  \spar(\bm{d}_i^t) - \bm{d}_i^t +  \bm{d}_i^t \right\|^2 + \frac{LkC^2\sigma^2}{2r^2}\\\nonumber
    & \leq \frac{L\eta_l^2\tau^2}{2n} \sum_{i=1}^{n}\mathbb{E}_t\left[2\left\| \spar(\bm{d}_i^t) - \bm{d}_i^t \right\|^2 + 2\left\| \bm{d}_i^t \right\|^2\right] + \frac{LC^2k\sigma^2}{2r^2}\\
    & \leq  \frac{L\eta_l^2\tau^2(\phi_k +1)}{n} \sum_{i=1}^{n} \mathbb{E}_t\left\|\bm{d}_i^t \right\|^2  + \frac{LC^2k\sigma^2}{2r^2}
\end{align}
where the first inequality uses the independence of Gaussian noise and Lemma~\ref{lemma:jensen}, the second inequality also uses Lemma~\ref{lemma:aaaa}, and the last inequality results from Lemma~\ref{lemma:bounded-sparsifier}. 
By Lemma~\ref{lemma:bounded_local_model}, $T_2$ is bounded as follows:
\begin{align}
\label{eqn:T_2_randk}
    T_2 & \leq 2L\eta_l^2\tau^2(\phi_k +1)\beta^2\left\|\nabla f(\bm{\theta}^{t})\right\|^2  + \frac{2L^3\eta_l^2\tau(\phi_k +1)}{n} \sum_{i=1}^{n}\sum_{s=0}^{\tau-1}\mathbb{E}_t\left\|\bm{\theta}^{t}-\bm{\theta}_i^{t,s}\right\|^2 + L\eta_l^2\tau^2(\phi_k +1)( \bar{\zeta}^2 + 2 \kappa^2)
    + \frac{LC^2k\sigma^2}{2r^2}
\end{align}

Combining \eqref{eqn:A}, \eqref{eqn:T_1_randk} and \eqref{eqn:T_2_randk}, one yields 
\begin{align}
\nonumber
    \mathbb{E}_t[f(\bm{\theta}^{t+1}) -f(\bm{\theta}^t)] & \leq \left( -\frac{\eta_l\tau}{2} +  2L\eta_l^2\tau^2(\phi_k +1)\beta^2\right) \left\|  \nabla f(\bm{\theta}^{t})\right\|^2 \\\nonumber
    & \ \ + \frac{\eta_l L^2 + 4L^3\eta_l^2\tau(\phi_k +1)}{2n} \sum_{i=1}^{n}\sum_{s=0}^{\tau-1} \mathbb{E}_t\left\| \bm{\theta}^{t}- \bm{\theta}_i^{t,s} \right\|^2 + L\eta_l^2\tau^2(\phi_k +1)( \bar{\zeta}^2 + 2 \kappa^2) + \frac{LC^2k\sigma^2}{2r^2}\\
    & \leq -\frac{\eta_l\tau}{4}\left\|  \nabla f(\bm{\theta}^{t})\right\|^2 + \frac{\eta_l L^2 (2\beta^2+1)}{4n\beta^2} \sum_{i=1}^{n}\sum_{s=0}^{\tau-1} \mathbb{E}_t\left\| \bm{\theta}^{t}- \bm{\theta}_i^{t,s} \right\|^2 + L\eta_l^2\tau^2(\phi_k +1)( \bar{\zeta}^2 + 2 \kappa^2)
    + \frac{LkC^2\sigma^2}{2r^2},
\end{align}
if we choose $\eta_l \leq {1}/(8\tau L (\phi_k +1)\beta^2)$. Next, by Lemma~\ref{lemma:bounded_local_divergence}, we obtain that 
\begin{align}
\nonumber
    \mathbb{E}_t[f(\bm{\theta}^{t+1}) -f(\bm{\theta}^t)] & \leq  -\frac{\eta_l\tau}{4}\left\|  \nabla f(\bm{\theta}^{t})\right\|^2 + \frac{\eta_l L^2 (2\beta^2+1)}{4\beta^2}\sum_{s=0}^{\tau-1} \left[16\eta_l^2\tau^2\beta^2\left\| \nabla f(\bm{\theta}^t)\right\|^2 + 16\eta_l^2\tau^2\kappa^2 + {4\tau\eta_l^2}\bar{\zeta}^2\right] \\\nonumber
    & \ \ + L\eta_l^2\tau^2(\phi_k +1)( \bar{\zeta}^2 + 2 \kappa^2)
    + \frac{LC^2k\sigma^2}{2r^2}\\\nonumber
    & = \left(-\frac{\eta_l\tau}{4}  + {4\eta_l L^2 (2\beta^2+1)\eta_l^2\tau^3} \right) \left\|\nabla f(\bm{\theta}^{t})\right\|^2 +  \frac{\eta_l^3\tau^2 L^2 (2\beta^2+1)\left(4\tau\kappa^2 + \bar{\zeta}^2\right)}{\beta^2}\\\nonumber
    & \ \ + L\eta_l^2\tau^2(\phi_k +1)( \bar{\zeta}^2 + 2 \kappa^2)   + \frac{LC^2k\sigma^2}{2r^2} \\\nonumber
    & \leq -\frac{\eta_l\tau}{8} \left\|\nabla f(\bm{\theta}^{t})\right\|^2 + {\eta_l^2\tau^2L(12\tau\eta_lL + 2)\kappa^2} + \eta_l^2\tau^2L(3\eta_l L + 1 ) \bar{\zeta}^2\\\nonumber
    & \ \ +  L\eta_l^2\tau^2\phi_k( \bar{\zeta}^2 + 2 \kappa^2)   + \frac{LC^2k\sigma^2}{2r^2}\\\label{eqn:final}
    & \leq -\frac{\eta_l\tau}{8} \left\|\nabla f(\bm{\theta}^{t})\right\|^2 + {\eta_l^2\tau^2L}(3\kappa^2 + 2\bar{\zeta}^2) +  \eta_l^2\tau^2L(2\kappa^2 + \bar{\zeta}^2)\phi_k   + \frac{LC^2k\sigma^2}{2r^2}
    % + \frac{\eta_l\left(8\tau\kappa^2 + \bar{\zeta}^2\right)}{64\beta^2} + 
    % \frac{\eta_l\tau(\bar{\zeta}^2 + 2 \kappa^2)}{8\beta^2} + L\eta_l^2\tau^2\phi_k( \bar{\zeta}^2 + 2 \kappa^2)   + \frac{LC^2\sigma^2}{2r^2}\\
    % & \leq  -\frac{\eta_l\tau}{8} \left\|\nabla f(\bm{\theta}^{t})\right\|^2 + \frac{\eta_l(3\tau\kappa^2 + (\tau+1)\bar{\zeta}^2)}{8\beta^2} + \eta_l^2\tau^2L\phi_k(\bar{\zeta}^2 + 2 \kappa^2)   + \frac{LC^2\sigma^2}{2r^2}
\end{align}
if ${ \eta_l} \leq \min\{ {1}/(4\tau L \sqrt{4\beta^2+2} ), 1/(12\tau L) \}$. Rearranging the above inequality in \eqref{eqn:final} and summing it form $t=0$ to $T-1$, we get
\begin{align}
\nonumber
    \sum_{t=0}^{T-1}\left\|\nabla f(\bm{\theta}^{t})\right\|^2 & \leq  \frac{8}{\eta_l\tau}  \mathbb{E}[\sum_{t=0}^{T-1}f(\bm{\theta}^{t}) -f(\bm{\theta}^{t+1})] + {8T}{\eta_l\tau L}(3\kappa^2 + 2\bar{\zeta}^2) +  {8T}\eta_l\tau L(2\kappa^2 + \bar{\zeta}^2)\phi_k   + \frac{4T LC^2k\sigma^2}{\eta_l\tau r^2}\\\label{eqn:final_randk}
    & \leq  \frac{8(f(\bm{\theta}^{0}) -f(\bm{\theta}^{*}))}{\eta_l\tau}   + {8T}{\eta_l\tau L}(3\kappa^2 + 2\bar{\zeta}^2) +  {8T}\eta_l\tau L(2\kappa^2 + \bar{\zeta}^2)\phi_k   + \frac{4T LC^2k\sigma^2}{\eta_l\tau r^2},
\end{align}
where the expectation is taken over all rounds $t\in[0,T-1]$. Dividing both sides of \eqref{eqn:final_randk} by $T$, one yields
\begin{align}
\nonumber
    \frac{1}{T}\sum_{t=0}^{T-1}\left\|\nabla f(\bm{\theta}^{t})\right\|^2 & \leq \frac{8(f(\bm{\theta}^{0}) -f^{*})}{T\eta_l\tau}   + {8}{\eta_l\tau L}(3\kappa^2 + 2\bar{\zeta}^2) +  {8}\eta_l\tau L(2\kappa^2 + \bar{\zeta}^2)\phi_k   + \frac{4LC^2k\sigma^2}{\eta_l\tau r^2},
\end{align}
if the learning rates $ \eta_l$ satisfy ${ \eta_l} \leq \min\{ {1}/(8\tau L (\phi_k +1)\beta^2),1/(4\tau L \sqrt{4\beta^2+2} ), 1/(12\tau L)  \}$. Here, we use the fact that $f^*\leq f(\bm{\theta}^{T})$.

%%%%%%%%%%%%%%%%%%%%%%%%%%%%%%%%%%%%%%%%%%%%%%%%%%%
\section{Proof of Theorem~\ref{thm:convergence} with $\topk_k$ sparsifier}\label{subsec:proof_conve_topk}
% \hl{see if can get rid of the global learning rate: no, global learning rate is important for model update perturbation}
Here, $\spar$ represents the operation of $\topk_k$ sparsification in Fed-SMP, and we assume that the distribution of the public distribution is similar to the overall distribution $\{\mathcal{D}_i\}_i\in[n]$. By the $L$-smoothness of function $f$, we have
\begin{align}
    \nonumber
    \mathbb{E}_t[f(\bm{\theta}^{t+1}) -f(\bm{\theta}^t)] & \leq  \mathbb{E}_t\left\langle \nabla f(\bm{\theta}^{t}),   \bm{\theta}^{t+1} - \bm{\theta}^t\right\rangle + \frac{L}{2}\mathbb{E}_t\left\| \bm{\theta}^{t+1} - \bm{\theta}^t\right\|^2
    \\\nonumber
    & = -\mathbb{E}_t\left\langle \nabla f(\bm{\theta}^{t}),   \mathbb{E}_{\mathcal{W}^t}\left[\frac{1}{r} \sum_{i\in\mathcal{W}^t} \Delta_i^t\right] \right\rangle + \frac{L}{2}\mathbb{E}_t\left\|  \frac{1}{r} \sum_{i\in\mathcal{W}^t} \Delta_i^t\right\|^2\\\nonumber
    & = -\mathbb{E}_t\left\langle \nabla f(\bm{\theta}^{t}),\frac{1}{n} \sum_{i=1}^{n}\Delta_i^t \right\rangle + \frac{L}{2}\mathbb{E}_t\left\|  \frac{1}{r} \sum_{i\in\mathcal{W}^t} \Delta_i^t\right\|^2\\\label{eqn:A_topk}
    & = \underbrace{-\mathbb{E}_t\left\langle \nabla f(\bm{\theta}^{t}),\frac{1}{n} \sum_{i=1}^{n}\left(\eta_l\tau\spar(\bm{d}_i^t) + \bm{b}_i^t\right) \right\rangle}_{T_1} + \underbrace{ \frac{L}{2}\mathbb{E}_t\left\|  \frac{1}{r} \sum_{i\in\mathcal{W}^t}\left( \eta_l\tau\spar(\bm{d}_i^t) + \bm{b}_i^t\right) \right\|^2}_{T_2}
\end{align}
where the expectation $\mathbb{E}_t[\cdot]$ is taken over the sampled clients $\mathcal{W}^t$ and mini-batches $\xi_i^s,\forall i\in[n], s\in\{0,\dots,\tau-1\}$ at round $t$. 
% \hl{(How does the client sampling effect show up in the proof? It seems that the proof does not really consider client sampling???)}. 
Due to the unbiasedness of the stochastic gradient and Gaussian noise, we have
\begin{align}
\nonumber
    T_1 & = -\left\langle \nabla f(\bm{\theta}^{t}),\mathbb{E}_t\left[\frac{1}{n} \sum_{i=1}^{n}\eta_l\tau\bm{d}_i^t  \right]\right\rangle 
    + \left\langle \nabla f(\bm{\theta}^{t}),\mathbb{E}_t\left[\frac{1}{n} \sum_{i=1}^{n}\eta_l\tau\bm{d}_i^t -\frac{1}{n} \sum_{i=1}^{n}\eta_l\tau\spar(\bm{d}_i^t) \right]\right\rangle 
    -\left\langle \nabla f(\bm{\theta}^{t}),   \mathbb{E}_t\left[\frac{1}{n} \sum_{i=1}^{n}\bm{b}_i^t \right]\right\rangle \\\label{eqn:T_1_topk} 
    & = \underbrace{-\eta_l\tau\mathbb{E}_t\left\langle \nabla f(\bm{\theta}^{t}),   \frac{1}{n} \sum_{i=1}^{n}\bm{h}_i^t  \right\rangle }_{A_1}
    + \underbrace{{\eta_l\tau} \left\langle \nabla f(\bm{\theta}^{t}),\mathbb{E}_t\left[\frac{1}{n}\sum_{i=1}^{n}\bm{d}_i^t -\spar\left(\frac{1}{n}\sum_{i=1}^{n}\bm{d}_i^t\right)\right]\right\rangle}_{A_2}
\end{align}
since the clients at each round use the same $\topk_k$ sparsifier. Using the fact that $2\langle a, b \rangle = \|a\|^2 + \|b\|^2 - \|a-b\|^2$, we have
\begin{align}
\nonumber
    A_1 &= -\frac{\eta_l\tau }{2}\left\|  \nabla f(\bm{\theta}^{t})\right\|^2 - \frac{\eta_l\tau}{2}\mathbb{E}_t\left\| \frac{1}{n} \sum_{i=1}^{n}\bm{h}_i^t  \right\|^2 
    + \frac{\eta_l\tau}{2} \mathbb{E}_t\left\| \nabla f(\bm{\theta}^{t}) -  \frac{1}{n} \sum_{i=1}^{n}\bm{h}_i^t \right\|^2 \\\nonumber
    &= -\frac{\eta_l\tau}{2}\left\|  \nabla f(\bm{\theta}^{t})\right\|^2 + \frac{\eta_l\tau}{2} \mathbb{E}_t\left\| \frac{1}{n} \sum_{i=1}^{n} \frac{1}{\tau}\sum_{s=0}^{\tau-1}\left( \nabla f_i(\bm{\theta}^{t}) - \nabla f_i(\bm{\theta}_i^{t,s}) \right) \right\|^2 - \frac{\eta_l\tau}{2}\mathbb{E}_t\left\| \frac{1}{n} \sum_{i=1}^{n}\bm{h}_i^t  \right\|^2  \\\nonumber
    & \leq -\frac{\eta_l\tau}{2}\left\|  \nabla f(\bm{\theta}^{t})\right\|^2 + \frac{\eta_l\tau}{2n} \sum_{i=1}^{n} \frac{1}{\tau}\sum_{s=0}^{\tau-1} \mathbb{E}_t\left\| \nabla f_i(\bm{\theta}^{t}) - \nabla f_i(\bm{\theta}_i^{t,s}) \right\|^2  - \frac{\eta_l\tau}{2}\mathbb{E}_t\left\| \frac{1}{n} \sum_{i=1}^{n}\bm{h}_i^t  \right\|^2 \\\label{eqn:A_1_top_k}
    & \leq -\frac{\eta_l\tau}{2}\left\|  \nabla f(\bm{\theta}^{t})\right\|^2 + \frac{\eta_l L^2}{2n} \sum_{i=1}^{n} \sum_{s=0}^{\tau-1} \mathbb{E}_t\left\| \bm{\theta}^{t}- \bm{\theta}_i^{t,s} \right\|^2  - \frac{\eta_l\tau}{2}\mathbb{E}_t\left\| \frac{1}{n} \sum_{i=1}^{n}\bm{h}_i^t  \right\|^2,
\end{align}
where the first inequality uses Lemma~\ref{lemma:aaaa}, and the second inequality uses the $L$-smoothness of function $f_i$. Next, let $\phi_k:= 1-k/d$, we get
\begin{align}
\nonumber
    A_2 & \leq \frac{\eta_l\tau}{2} \left(\gamma\left\|\nabla f(\bm{\theta}^{t}) \right\|^2 + \gamma^{-1} \mathbb{E}_t\left\| \frac{1}{n}\sum_{i=1}^{n}\bm{d}_i^t -\spar\left(\frac{1}{n}\sum_{i=1}^{n}\bm{d}_i^t\right) \right\|^2\right)\\\nonumber
    & \leq \frac{\eta_l\tau\gamma}{2}\left\|\nabla f(\bm{\theta}^{t}) \right\|^2  + \frac{\eta_l\tau\phi_k}{2\gamma} \mathbb{E}_t\left\|\frac{1}{n}\sum_{i=1}^{n}\bm{d}_i^t \right\|^2\\\nonumber
    & = \frac{\eta_l\tau\gamma}{2}\left\|\nabla f(\bm{\theta}^{t}) \right\|^2  + \frac{\eta_l\tau\phi_k}{2\gamma} \mathbb{E}_t\left\|\frac{1}{n}\sum_{i=1}^{n}\bm{d}_i^t - \frac{1}{n}\sum_{i=1}^{n}\bm{h}_i^t + \frac{1}{n}\sum_{i=1}^{n}\bm{h}_i^t \right\|^2\\\nonumber
    & = \frac{\eta_l\tau\gamma}{2}\left\|\nabla f(\bm{\theta}^{t}) \right\|^2  + \frac{\eta_l\tau\phi_k}{2\gamma}  \mathbb{E}_t\left[\left\| \frac{1}{n}\sum_{i=1}^{n}\bm{d}_i^t -  \frac{1}{n}\sum_{i=1}^{n}\bm{h}_i^t\right\|^2 + \left\| \frac{1}{n}\sum_{i=1}^{n}\bm{h}_i^t \right\|^2 +2\left\langle \frac{1}{n}\sum_{i=1}^{n}\bm{h}_i^t,  \frac{1}{n}\sum_{i=1}^{n}\bm{d}_i^t -  \frac{1}{n}\sum_{i=1}^{n}\bm{h}_i^t\right\rangle\right] \\\label{eqn:A_2_top_k}
    & = \frac{\eta_l\tau\gamma}{2}\left\|\nabla f(\bm{\theta}^{t}) \right\|^2  + \frac{\eta_l\tau\phi_k}{2\gamma}  \mathbb{E}_t\left[\left\| \frac{1}{n}\sum_{i=1}^{n}\bm{d}_i^t -  \frac{1}{n}\sum_{i=1}^{n}\bm{h}_i^t\right\|^2 + \left\| \frac{1}{n}\sum_{i=1}^{n}\bm{h}_i^t \right\|^2 \right]\\
    & \leq \frac{\eta_l\tau\gamma}{2}\left\|\nabla f(\bm{\theta}^{t}) \right\|^2  + \frac{\eta_l\tau\phi_k}{2n\gamma}  \sum_{i=1}^{n}\mathbb{E}_t\left\| \bm{d}_i^t - \bm{h}_i^t\right\|^2 + \frac{\eta_l\tau\phi_k}{2\gamma} \mathbb{E}_t\left\| \frac{1}{n}\sum_{i=1}^{n}\bm{h}_i^t \right\|^2,
\end{align}
where the first inequality uses Lemma~\ref{lemma:ab}, the second inequality uses Lemma~\ref{lemma:bounded-sparsifier}, and the third inequality uses Lemma~\ref{lemma:aaaa}. Given that 
\begin{align}
    \mathbb{E}_t\left\| \bm{d}_i^t -\bm{h}_i^t\right\|^2 = \mathbb{E}_t\left\| \frac{1}{\tau}\sum_{s=0}^{\tau-1}\left(\bm{g}_i^{t,s} - \nabla f_i(\bm{\theta}_i^{t,s})\right)\right\|^2 \leq \frac{1}{\tau}\sum_{s=0}^{\tau-1}\mathbb{E}_t\left\| \bm{g}_i^{t,s} - \nabla f_i(\bm{\theta}_i^{t,s})\right\|^2 \leq \frac{1}{\tau}\sum_{s=0}^{\tau-1} \zeta_i^2 =  \zeta_i^2,
\end{align}
by Lemma~\ref{lemma:aaaa} and Assumption~\ref{assp:bounded_variance}, we have
\begin{align}
    A_2 & \leq \frac{\eta_l\tau\gamma}{2}\left\|\nabla f(\bm{\theta}^{t}) \right\|^2  + \frac{\eta_l\tau\phi_k}{2n\gamma}  \sum_{i=1}^{n}\zeta_i^2 +\frac{\eta_l\tau\phi_k}{2\gamma}  \mathbb{E}_t\left\| \frac{1}{n}\sum_{i=1}^{n}\bm{h}_i^t \right\|^2. 
\end{align}
Combining \eqref{eqn:T_1_topk}, \eqref{eqn:A_1_top_k} and \eqref{eqn:A_2_top_k} and let $\bar{\zeta}^2:=(1/n)\sum_{i=1}^{n}\zeta_i^2$, we get 
\begin{align}
\nonumber
    T_1 & \leq -\frac{\eta_l\tau (1- \gamma)}{2}\left\|  \nabla f(\bm{\theta}^{t})\right\|^2 + \frac{\eta_l\tau L^2}{2n} \sum_{i=1}^{n} \frac{1}{\tau}\sum_{s=0}^{\tau-1} \mathbb{E}_t\left\| \bm{\theta}^{t}- \bm{\theta}_i^{t,s} \right\|^2 - \frac{\eta_l\tau (1-\phi_k/\gamma)}{2}\mathbb{E}_t\left\| \frac{1}{n} \sum_{i=1}^{n}\bm{h}_i^t  \right\|^2 + \frac{\eta_l\tau\phi_k\bar{\zeta}^2}{2\gamma}\\\label{eqn:new_T_1_topk}
    & \leq  -\frac{\eta_l\tau (1- \gamma)}{2}\left\|  \nabla f(\bm{\theta}^{t})\right\|^2 + \frac{\eta_l L^2}{2n} \sum_{i=1}^{n}\sum_{s=0}^{\tau-1} \mathbb{E}_t\left\| \bm{\theta}^{t}- \bm{\theta}_i^{t,s} \right\|^2 + \frac{\eta_l\tau\phi_k\bar{\zeta}^2}{2\gamma},
\end{align}
if $\phi_k \leq \gamma $. 

To bound $T_2$, we utilize the independence of Gaussian noise and obtain that
\begin{align}
\nonumber
   T_2 & = \frac{L\eta_l^2\tau^2}{2}\mathbb{E}_t\left\| \spar\left( \frac{1}{r} \sum_{i\in\mathcal{W}^t} \bm{d}_i^t\right)  \right\|^2 +  \frac{L}{2}\mathbb{E}_t\left\|  \frac{1}{r} \sum_{i\in\mathcal{W}^t}\bm{b}_i^t\right\|^2 \\\nonumber
   & =  \frac{L\eta_l^2\tau^2}{2}\mathbb{E}_t\left\| \spar\left( \frac{1}{r} \sum_{i\in\mathcal{W}^t} \bm{d}_i^t\right) - \frac{1}{r} \sum_{i\in\mathcal{W}^t} \bm{d}_i^t +  \frac{1}{r} \sum_{i\in\mathcal{W}^t} \bm{d}_i^t \right\|^2 +  \frac{LkC^2\sigma^2}{2r^2}\\\nonumber
    & \leq  {L\eta_l^2\tau^2}\mathbb{E}_t\left\| \spar\left( \frac{1}{r} \sum_{i\in\mathcal{W}^t} \bm{d}_i^t\right) - \frac{1}{r} \sum_{i\in\mathcal{W}^t} \bm{d}_i^t  \right\|^2 + {L\eta_l^2\tau^2}\mathbb{E}_t\left\| \frac{1}{r} \sum_{i\in\mathcal{W}^t} \bm{d}_i^t \right\|^2 +  \frac{LkC^2\sigma^2}{2r^2}\\\nonumber
   & \leq  {L\eta_l^2\tau^2(\phi_k+1)}\mathbb{E}_t\left\| \frac{1}{r} \sum_{i\in\mathcal{W}^t} \bm{d}_i^t  \right\|^2 +  \frac{LkC^2\sigma^2}{2r^2}\\\nonumber
   & \leq {L\eta_l^2\tau^2(\phi_k+1)}\mathbb{E}_t\left[  \frac{1}{r} \sum_{i\in\mathcal{W}^t}\left\| \bm{d}_i^t  \right\|^2\right] +  \frac{LkC^2\sigma^2}{2r^2}\\
   & = {L\eta_l^2\tau^2(\phi_k+1)}\mathbb{E}_t\left[\frac{1}{n} \sum_{i=1}^{n}\left\| \bm{d}_i^t  \right\|^2\right] +  \frac{LkC^2\sigma^2}{2r^2},
\end{align}
where the first inequality uses the fact that $\|a+b\|^2 \leq 2\|a\|^2 + 2\|b\|^2$, the second inequality uses Lemma~\ref{lemma:bounded-sparsifier}, and the last inequality results from Lemma~\ref{lemma:jensen}. Then, by Lemma~\ref{lemma:bounded_local_model}, we get 
\begin{align}
\label{eqn:new_T_2_topk}
    T_2 & \leq {L\eta_l^2\tau^2(\phi_k+1)}\left( \frac{2L^2}{n\tau} \sum_{i=1}^{n} \sum_{s=0}^{\tau-1}\mathbb{E}_t\left\|\bm{\theta}^{t}-\bm{\theta}_i^{t,s}\right\|^2 + 2(\beta^2\left\|\nabla f(\bm{\theta}^{t})\right\|^2 + \kappa^2) + \bar{\zeta}^2 \right) +  \frac{LkC^2\sigma^2}{2r^2}
\end{align}

Combining \eqref{eqn:A_topk}, \eqref{eqn:new_T_1_topk} and \eqref{eqn:new_T_2_topk}, one yields
\begin{align}
\nonumber
    \mathbb{E}_t[f(\bm{\theta}^{t+1}) -f(\bm{\theta}^t)] & \leq  \left( -\frac{\eta_l\tau (1- \gamma)}{2} + 2{L\eta_l^2\tau^2(\phi_k+1)}\beta^2 \right)\left\|  \nabla f(\bm{\theta}^{t})\right\|^2 \\\nonumber
    & \ \ + \left(\frac{\eta_l L^2}{2n} + \frac{{2L\eta_l^2\tau^2(\phi_k+1)}L^2}{n\tau} \right) \sum_{i=1}^{n}\sum_{s=0}^{\tau-1} \mathbb{E}_t\left\| \bm{\theta}^{t}- \bm{\theta}_i^{t,s} \right\|^2 + \frac{\eta_l\tau\phi_k\bar{\zeta}^2}{2\gamma} \\\nonumber
    & \ \ + {L\eta_l^2\tau^2(\phi_k+1)}(2\kappa^2+\bar{\zeta}^2) +  \frac{LkC^2\sigma^2}{2r^2}\\\nonumber
    & \leq -\frac{\eta_l\tau}{4} \left\|  \nabla f(\bm{\theta}^{t})\right\|^2 + \frac{\eta_l L^2(2\beta^2+1)}{4n\beta^2} \sum_{i=1}^{n}\sum_{s=0}^{\tau-1} \mathbb{E}_t\left\| \bm{\theta}^{t}- \bm{\theta}_i^{t,s} \right\|^2 + \frac{\eta_l\tau\phi_k\bar{\zeta}^2}{2\gamma} \\
    & \ \ + {L\eta_l^2\tau^2(\phi_k+1)}(2\kappa^2+\bar{\zeta}^2) +  \frac{LkC^2\sigma^2}{2r^2}
    % &  32\eta_l^2\tau^2\beta^2\left\| \nabla f(\bm{\theta}^t)\right\|^2 + 32\eta_l^2\tau^2\kappa^2 + {4\tau\eta_l^2}\bar{\zeta}^2.
\end{align}
if $\eta_l \leq (1-2\gamma) /(8\tau L (\phi_k+1)\beta^2)$ and $\gamma < 1/2$. Then, by Lemma~\ref{lemma:bounded_local_divergence}, we have
\begin{align}
\nonumber
    \mathbb{E}_t[f(\bm{\theta}^{t+1}) -f(\bm{\theta}^t)] & \leq -\frac{\eta_l\tau}{4} \left\|  \nabla f(\bm{\theta}^{t})\right\|^2 + \frac{\eta_l L^2(2\beta^2+1)}{4\beta^2} \sum_{s=0}^{\tau-1} \left[16\eta_l^2\tau^2\beta^2\left\| \nabla f(\bm{\theta}^t)\right\|^2 + 16\eta_l^2\tau^2\kappa^2 + {4\tau\eta_l^2}\bar{\zeta}^2\right] + \frac{\eta_l\tau\phi_k\bar{\zeta}^2}{2\gamma} \\\nonumber
    & \ \ + {L\eta_l^2\tau^2(\phi_k+1)}(2\kappa^2+\bar{\zeta}^2) +  \frac{LkC^2\sigma^2}{2r^2}\\\nonumber
    & \leq \left(-\frac{\eta_l\tau}{4} + {4\eta_l^3\tau^3 L^2(2\beta^2+1)} \right) \left\|  \nabla f(\bm{\theta}^{t})\right\|^2 + \frac{\eta_l^3\tau^2 L^2(2\beta^2+1)(4\tau\kappa^2 + \bar{\zeta}^2)}{\beta^2}  + \frac{\eta_l\tau\phi_k\bar{\zeta}^2}{2\gamma} \\\nonumber
    & \ \ + {L\eta_l^2\tau^2(\phi_k+1)}(2\kappa^2+\bar{\zeta}^2) +  \frac{LkC^2\sigma^2}{2r^2}\\\nonumber
    & \leq -\frac{\eta_l\tau}{8}\left\|  \nabla f(\bm{\theta}^{t})\right\|^2 +  \eta_l^2\tau^2 L (12\eta_l\tau L + 2) \kappa^2 + \eta_l^2\tau^2 L (3\eta_l L + 1)\bar{\zeta}^2  \\\nonumber
    & \ \ + \left({L\eta_l^2\tau^2}(2\kappa^2+\bar{\zeta}^2) + \frac{\eta_l\tau\bar{\zeta}^2}{2\gamma} \right)\phi_k +  \frac{LkC^2\sigma^2}{2r^2}\\\label{eqn:final_topk}
     & \leq -\frac{\eta_l\tau}{8}\left\|  \nabla f(\bm{\theta}^{t})\right\|^2 +  \eta_l^2\tau^2 L (3\kappa^2 + 2\bar{\zeta}^2) + \left({L\eta_l^2\tau^2}(2\kappa^2+\bar{\zeta}^2) + \frac{\eta_l\tau\bar{\zeta}^2}{2\gamma} \right)\phi_k  +  \frac{LkC^2\sigma^2}{2r^2},
\end{align}
if ${ \eta_l} \leq \min\{ {1}/(4\tau L \sqrt{4\beta^2+2} ), 1/(12\tau L) \}$. Rearranging the above inequality in \eqref{eqn:final_topk} and summing it form $t=0$ to $T-1$, we get
\begin{align}
\label{eqn:final_v2_topk}
    \sum_{t=0}^{T-1}\left\|  \nabla f(\bm{\theta}^{t})\right\|^2 & \leq \frac{8}{\eta_l\tau} \sum_{t=0}^{T-1}\mathbb{E}_t[f(\bm{\theta}^{t+1}) -f(\bm{\theta}^t)]  +  8T\eta_l\tau L (3\kappa^2 + 2\bar{\zeta}^2) + {8T}\left({\eta_l\tau L}(2\kappa^2+\bar{\zeta}^2) + \frac{\bar{\zeta}^2}{2\gamma} \right)\phi_k  + \frac{4TLkC^2\sigma^2}{\eta_l\tau r^2},
\end{align}
where the expectation is taken over all rounds $t\in[0,T-1]$. Dividing both sides of \eqref{eqn:final_v2_topk} by $T$, one yields
\begin{align*}
    \frac{1}{T}\sum_{t=0}^{T-1}\left\|  \nabla f(\bm{\theta}^{t})\right\|^2 & \leq \frac{8(f(\bm{\theta}^{0}) -f^*)}{T\eta_l\tau}   +  8\eta_l\tau L (3\kappa^2 + 2\bar{\zeta}^2) + {8}\left({\eta_l\tau L}(2\kappa^2+\bar{\zeta}^2) + \frac{\bar{\zeta}^2}{2\gamma} \right)\phi_k  + \frac{4LkC^2\sigma^2}{\eta_l\tau r^2},
\end{align*}
Here, we use the fact that $f^*\leq f(\bm{\theta}^{T})$. We summarize the convergence results of Fed-SMP with $\randk_k$ sparsifier and $\topk_k$ sparsifier in Theorem~\ref{thm:convergence} by selecting a reasonable constant $\gamma = 1/3$ (which satisfies $\gamma < 1/2$). % {\color{red}{why the above formula still has \gamma}}.

\section{Intermediate Results}
%---------------------%
%%%%%%%%%%%%%%%%%%%%%%%%%%%%%%%%%%%%%%%%%%%%%%%%%%%%%%
\begin{lemma}[Bounded Sparsification]\label{lemma:bounded-sparsifier}
Given a vector $\mathbf{x}\in\mathbb{R}^d$, parameter $k \in [d]$. The sparsifier $\{ \randk_k(\mathbf{x}), \topk_k(\mathbf{x})\}$ holds that 
\begin{equation*}
 \mathbb{E}\|\topk_k(\mathbf{x})-\mathbf{x}\|^2 \leq \left(1-\frac{k}{d}\right)\|\mathbf{x}\|^2; \quad \mathbb{E}\left\|\frac{d}{k}\randk_k(\mathbf{x})-\mathbf{x}\right\|^2 \leq \left(\frac{d}{k}-1\right)\|\mathbf{x}\|^2
\end{equation*}
\end{lemma}
%---------------------%
\begin{proof}
For the $\randk_k$ sparsifier, by applying the expectation over the active set $\omega$, we have
\begin{align*}
    \mathbb{E}_{\omega}[\frac{d}{k}\randk_k(\mathbf{x})] =\frac{d}{k}\left[\frac{k}{d}[\mathbf{x}]_1, \dots, \frac{k}{d}[\mathbf{x}]_d\right]=\mathbf{x},
\end{align*}
,
\begin{align*}
   \mathbb{E}_{\omega}\|\randk_k(\mathbf{x})-\mathbf{x}\|^2 & = \sum_{j=1}^{d}\left( \frac{k}{d} \left([\mathbf{x}]_j - [\mathbf{x}]_j\right)^2 +  \left(1-\frac{k}{d}\right)[\mathbf{x}]_j^2\right)\\
   %& = \left(\frac{d}{k}-1\right)\sum_{j=1}^{d}(\bm{\theta})_j^2\\
   & = \left(1-\frac{k}{d}\right)\|\mathbf{x}\|^2,
\end{align*}
and 
\begin{align*}
   \mathbb{E}_{\omega}\left\|\frac{d}{k}\randk_k(\mathbf{x})-\mathbf{x}\right\|^2 & = \sum_{j=1}^{d}\left( \frac{k}{d} \left(\frac{d}{k}[\mathbf{x}]_j - [\mathbf{x}]_j\right)^2 +  \left(1-\frac{k}{d}\right)[\mathbf{x}]_j^2\right)\\
   & = \left( \frac{k}{d} \left(\frac{d}{k} - 1\right)^2 +  \left(1-\frac{k}{d}\right)\right)\sum_{j=1}^{d}[\mathbf{x}]_j^2\\
   & = \left( \frac{d}{k} - 1\right)\|\mathbf{x}\|^2,
\end{align*}

For the $\topk_k$ sparsifier, as $\mathbb{E}\|\topk_k(\mathbf{x})-\mathbf{x}\|^2 \leq \mathbb{E}\|\randk_k(\mathbf{x})-\mathbf{x}\|^2$, we have
\begin{align*}
    \mathbb{E}\|\topk_k(\mathbf{x})-\mathbf{x}\|^2 \leq \left(1-\frac{k}{d}\right)\|\mathbf{x}\|^2.
\end{align*}
\end{proof}

%%%%%%%%%%%%%%%%%%%%%%%%%%%%%%%%%%%%%%%%%%%%%%%%%%%%%%
\begin{lemma}[Bounded Local Divergence]\label{lemma:bounded_local_divergence}
The local model difference at round $t$ is bounded as follows:
\begin{align}
    \frac{1}{n}\sum_{i=1}^{n}\mathbb{E}_t\left\|\bm{\theta}^{t} -  \bm{\theta}_i^{t,s}\right\|^2 \leq  32\eta_l^2\tau^2\beta^2\left\| \nabla f(\bm{\theta}^t)\right\|^2 + 32\eta_l^2\tau^2\kappa^2 + {4\tau\eta_l^2}\bar{\zeta}^2.
\end{align}
where $\bar{\zeta}^2 := (1/n)\sum_{i=1}^{n}\zeta_i^2$.
\end{lemma}
\begin{proof}
Plugging into the local update rule, we have
\begin{align}
    \nonumber
    \mathbb{E}_t\left\|\bm{\theta}^{t} -  \bm{\theta}_i^{t,s}\right\|^2 & = \mathbb{E}_t\left\| \bm{\theta}_i^{t,s-1} - \bm{\theta}^{t} - \eta_l \bm{g}_i^{t,s-1} \right\|^2 \\\nonumber
    & = \mathbb{E}_t\left\| \bm{\theta}_i^{t,s-1} - \bm{\theta}^{t} - \eta_l \bm{g}_i^{t,s-1} + \eta_l \nabla f_i( \bm{\theta}_i^{t,s-1}) -\eta_l \nabla f_i( \bm{\theta}_i^{t,s-1}) + \eta_l \nabla f_i(\bm{\theta}^t) - \eta_l \nabla f_i(\bm{\theta}^t) \right\|^2\\\nonumber
    & =  \mathbb{E}_t\left\| \bm{\theta}_i^{t,s-1} - \bm{\theta}^{t}  -\eta_l \nabla f_i( \bm{\theta}_i^{t,s-1}) + \eta_l \nabla f_i(\bm{\theta}^t) - \eta_l \nabla f_i(\bm{\theta}^t) \right\|^2 + \eta_l^2 \mathbb{E}_t\left\| \bm{g}_i^{t,s-1} - \nabla f_i( \bm{\theta}_i^{t,s-1})\right\|^2\\\nonumber
    & \leq \left(1+\frac{1}{2\tau-1}\right) \mathbb{E}_t\left\| \bm{\theta}_i^{t,s-1} - \bm{\theta}^{t}\right\|^2 +  2\eta_l^2\tau\mathbb{E}_t\left\|\nabla f_i( \bm{\theta}_i^{t,s-1}) +\nabla f_i(\bm{\theta}^t) - \nabla f_i(\bm{\theta}^t) \right\|^2 + \eta_l^2\zeta_i^2\\\nonumber
    & \leq \left(1+\frac{1}{2\tau-1}\right) \mathbb{E}_t\left\| \bm{\theta}_i^{t,s-1} - \bm{\theta}^{t}\right\|^2 +  4\eta_l^2\tau\mathbb{E}_t\left\|\nabla f_i( \bm{\theta}_i^{t,s-1}) - \nabla f_i(\bm{\theta}^t)\right\|^2  + 4\eta_l^2\tau\left\|\nabla f_i(\bm{\theta}^t) \right\|^2 +  \eta_l^2\zeta_i^2\\\nonumber
    & \leq \left(1+\frac{1}{2\tau-1}\right) \mathbb{E}_t\left\| \bm{\theta}_i^{t,s-1} - \bm{\theta}^{t}\right\|^2 +  4\eta_l^2L^2\tau\mathbb{E}_t\left\| \bm{\theta}_i^{t,s-1} - \bm{\theta}^t\right\|^2  + 4\eta_l^2\tau\left\|\nabla f_i(\bm{\theta}^t)\right\|^2  +  \eta_l^2\zeta_i^2,
\end{align}
by using Lemma~\ref{lemma:a+b} and Assumption~\ref{assp:bounded_variance}, and hence, 
\begin{align}
    \nonumber
    \frac{1}{n}\sum_{i=1}^{n}\mathbb{E}_t\left\|\bm{\theta}^{t} -  \bm{\theta}_i^{t,s}\right\|^2 & \leq \left(1+\frac{1}{2\tau-1} + 4\eta_l^2L^2\tau\right) \frac{1}{n}\sum_{i=1}^{n}\mathbb{E}_t\left\| \bm{\theta}_i^{t,s-1} - \bm{\theta}^{t}\right\|^2  + \frac{4\eta_l^2\tau}{n}\sum_{i=1}^{n}\mathbb{E}_t\left\|\nabla f_i(\bm{\theta}^t) \right\|^2  +  \frac{\eta_l^2}{n}\sum_{i=1}^{n}\zeta_i^2\\\nonumber
    & \leq \left(1+\frac{1}{2\tau-1} + 4\eta_l^2L^2\tau\right) \frac{1}{n}\sum_{i=1}^{n}\mathbb{E}_t\left\| \bm{\theta}_i^{t,s-1} - \bm{\theta}^{t}\right\|^2  + 4\eta_l^2\tau\beta^2\left\| \nabla f(\bm{\theta}^t)\right\|^2 + 4\eta_l^2\tau\kappa^2 + \frac{\eta_l^2}{n}\sum_{i=1}^{n}\zeta_i^2\\
    & \leq \left(1+\frac{1}{\tau-1}\right) \frac{1}{n}\sum_{i=1}^{n}\mathbb{E}_t\left\| \bm{\theta}_i^{t,s-1} - \bm{\theta}^{t}\right\|^2  + 4\eta_l^2\tau\beta^2\left\| \nabla f(\bm{\theta}^t)\right\|^2 + 4\eta_l^2\tau\kappa^2 + \frac{\eta_l^2}{n}\sum_{i=1}^{n}\zeta_i^2,
\end{align}
when $\eta_l\leq 1/3\tau L$. Unrolling the recursion, we obtain the following:
\begin{align}
\nonumber
    \frac{1}{n}\sum_{i=1}^{n}\mathbb{E}_t\left\|\bm{\theta}^{t} - \bm{\theta}_i^{t,s}\right\|^2 & \leq \sum_{h=0}^{s-1}  \left(1+\frac{1}{\tau-1}\right)^h\left[4\eta_l^2\tau\beta^2\left\| \nabla f(\bm{\theta}^t)\right\|^2 + 4\eta_l^2\tau\kappa^2 + \frac{\eta_l^2}{n}\sum_{i=1}^{n}\zeta_i^2\right]\\\nonumber
    & \leq (\tau-1)\left[\left(1+\frac{1}{\tau-1}\right)^\tau -1\right] \times \left[4\eta_l^2\tau\beta^2\left\| \nabla f(\bm{\theta}^t)\right\|^2 + 4\eta_l^2\tau\kappa^2 + \frac{\eta_l^2}{n}\sum_{i=1}^{n}\zeta_i^2\right]\\\label{eqn:t4-3}
    & \leq 16\eta_l^2\tau^2\beta^2\left\| \nabla f(\bm{\theta}^t)\right\|^2 + 16\eta_l^2\tau^2\kappa^2 + \frac{4\tau\eta_l^2}{n}\sum_{i=1}^{n}\zeta_i^2,
\end{align}
where the last inequality results from the fact that $\left(1+\frac{1}{\tau-1}\right)^\tau \leq 5 $ when $\tau>1$.
\end{proof}

%%%%%%%%%%%%%%%%%%%%%%%%%%%%%%%%%%%%%%%%%%%%%%%%%%%%%%
\begin{lemma}[Bounded Local Model Update]\label{lemma:bounded_local_model}
The local model update $\bm{d}_i^t$ at round $t$ is bounded as follows:
\begin{align}
     \frac{1}{n} \sum_{i=1}^{n}\mathbb{E}_t\left\|\bm{d}_i^t\right\|^2 & \leq  \frac{2L^2}{n\tau} \sum_{i=1}^{n} \sum_{s=0}^{\tau-1}\mathbb{E}_t\left\|\bm{\theta}^{t}-\bm{\theta}_i^{t,s}\right\|^2 + 2(\beta^2\left\|\nabla f(\bm{\theta}^{t})\right\|^2 + \kappa^2) + \bar{\zeta}^2 .
\end{align}
where $\bar{\zeta}^2 := (1/n)\sum_{i=1}^{n}\zeta_i^2$.
\end{lemma}
\begin{proof}
Since $ \mathbb{E}_t[\bm{d}_i^t]=\bm{h}_i^t$, we have
\begin{align}
\nonumber
  \frac{1}{n} \sum_{i=1}^{n}\mathbb{E}_t\left\|\bm{d}_i^t\right\|^2 & =\frac{1}{n} \sum_{i=1}^{n}\mathbb{E}_t\left\|\bm{d}_i^t -\bm{h}_i^t + \bm{h}_i^t\right\|^2\\\nonumber
  & = \frac{1}{n} \sum_{i=1}^{n}\left[\mathbb{E}_t\left\| \bm{d}_i^t -\bm{h}_i^t\right\|^2 + \mathbb{E}_t\left\| \bm{h}_i^t\right\|^2 + \mathbb{E}_t\left\langle\bm{d}_i^t -\bm{h}_i^t,  \bm{h}_i^t\right\rangle\right]\\\nonumber
    & = \frac{1}{n} \sum_{i=1}^{n}\left[\mathbb{E}_t\left\| \bm{d}_i^t -\bm{h}_i^t\right\|^2 + \mathbb{E}_t\left\| \bm{h}_i^t\right\|^2\right]\\\nonumber
    &= \frac{1}{n} \sum_{i=1}^{n}\left[\mathbb{E}_t\left\| \bm{d}_i^t -\bm{h}_i^t\right\|^2 + \mathbb{E}_t\left\| \bm{h}_i^t - \nabla f_i(\bm{\theta}^{t}) + \nabla f_i(\bm{\theta}^{t})\right\|^2\right]\\\nonumber
    & \leq \frac{1}{n} \sum_{i=1}^{n}\left[ \mathbb{E}_t\left\| \bm{d}_i^t -\bm{h}_i^t\right\|^2 + 2 \mathbb{E}_t\left\| \bm{h}_i^t - \nabla f_i(\bm{\theta}^{t})\right\|^2 \right] + \frac{2}{n} \sum_{i=1}^{n}\mathbb{E}_t\left\|\nabla f_i(\bm{\theta}^{t})\right\|^2\\
    & \leq \frac{1}{n} \sum_{i=1}^{n}\left[ \mathbb{E}_t\left\| \bm{d}_i^t -\bm{h}_i^t\right\|^2 + 2 \mathbb{E}_t\left\| \bm{h}_i^t - \nabla f_i(\bm{\theta}^{t})\right\|^2\right] + 2(\beta^2\left\|\nabla f(\bm{\theta}^{t})\right\|^2 + \kappa^2)
\end{align}
where the first inequality uses Lemma~\ref{lemma:aaaa}, and the last inequality uses Assumption~\ref{assp:bounded_divergence}. Given that 
\begin{align}
    \mathbb{E}_t\left\| \bm{d}_i^t -\bm{h}_i^t\right\|^2 = \mathbb{E}_t\left\| \frac{1}{\tau}\sum_{s=0}^{\tau-1}\left(\bm{g}_i^{t,s} - \nabla f_i(\bm{\theta}_i^{t,s})\right)\right\|^2 \leq \frac{1}{\tau}\sum_{s=0}^{\tau-1}\mathbb{E}_t\left\| \bm{g}_i^{t,s} - \nabla f_i(\bm{\theta}_i^{t,s})\right\|^2 \leq \frac{1}{\tau}\sum_{s=0}^{\tau-1} \zeta_i^2 =  \zeta_i^2,
\end{align}
by Lemma~\ref{lemma:aaaa} and Assumption~\ref{assp:bounded_variance}, and 
\begin{align}
    \mathbb{E}_t\left\| \bm{h}_i^t - \nabla f_i(\bm{\theta}^{t})\right\|^2= \mathbb{E}_t\left\| \frac{1}{\tau}\sum_{s=0}^{\tau-1} \left(\nabla f_i(\bm{\theta}_i^{t,s}) -\nabla f_i(\bm{\theta}^{t}) \right)\right\|^2 \leq  \frac{1}{\tau}\sum_{s=0}^{\tau-1}\mathbb{E}_t\left\| \nabla f_i(\bm{\theta}_i^{t,s}) -\nabla f_i(\bm{\theta}^{t})\right\|^2 \leq \frac{L^2}{\tau}\sum_{s=0}^{\tau-1}\mathbb{E}_t\left\|\bm{\theta}^{t}-\bm{\theta}_i^{t,s}\right\|^2
\end{align}
by Lemma~\ref{lemma:aaaa} and the $L$-smoothness of function $f_i$, one yields 
\begin{align}
   \frac{1}{n} \sum_{i=1}^{n}\mathbb{E}_t\left\|\bm{d}_i^t\right\|^2 & \leq  \bar{\zeta}^2 + \frac{2}{n} \sum_{i=1}^{n} \frac{L^2}{\tau}\sum_{s=0}^{\tau-1}\mathbb{E}_t\left\|\bm{\theta}^{t}-\bm{\theta}_i^{t,s}\right\|^2 + 2(\beta^2\left\|\nabla f(\bm{\theta}^{t})\right\|^2 + \kappa^2).
\end{align}
\end{proof}

\end{document}